%% file: main.tex
\documentclass[11pt]{article}
\usepackage[final]{acl}
\input{macros}

\title{\papertitle}
\author{
  \textbf{Hoang Pham}\textsuperscript{1} \quad
  \textbf{Duc Anh Vu}\textsuperscript{1} \quad
  \textbf{Duy Tung Doan}\textsuperscript{2} \\
  \textbf{Cong-Duy Nguyen}\textsuperscript{3} \quad
  \textbf{Anh Tuan Luu}\textsuperscript{1,3}\thanks{Corresponding author.} \\
  \textsuperscript{1}Nanyang Technological University \quad
  \textsuperscript{2}Viettel AI, Viettel Group \quad
  \textsuperscript{3}VinUniversity 
  \vspace{0.3cm}\\
  \texttt{Email:\hspace{0.4em}pham0100@ntu.edu.sg}
}

\begin{document}
\maketitle
\input{sections/00_abstract}
\input{sections/01_intro}
\input{sections/02_related_work}
\input{sections/03_setup}
\input{sections/04_analysis}
\input{sections/05_method}
\input{sections/06_experiments}
\input{sections/07_conclusion}
\input{sections/08_limitations}

\bibliography{refs}

\appendix
\input{sections/A_appendix}

\end{document}

%% file: macros.tex
\usepackage[T1]{fontenc}
\usepackage[utf8]{inputenc}
\usepackage{times}
\usepackage{microtype}
\usepackage{amsmath,amssymb,amsthm}
\newtheorem{proposition}{Proposition}

\newtheorem{lemma}{Lemma}
\theoremstyle{remark}
\newtheorem{remark}{Remark}
\usepackage{booktabs}
\usepackage{multirow}
\usepackage{pifont}
\usepackage{enumitem}
\usepackage{graphicx}
\usepackage{adjustbox}
\usepackage{float}
\usepackage{placeins}
\usepackage[table]{xcolor}
\usepackage{tikz}
\usetikzlibrary{positioning,arrows.meta,fit,calc}
\usepackage{algorithm}
\usepackage{algpseudocode}

\newcommand{\papertitle}{Copying Is Not Grounding: Training Faithful
Long-Form RAG\\with Authored-Sentence Credit}

\definecolor{pblueD}{HTML}{1A3A5F}   %
\definecolor{pblue}{HTML}{2C5F8E}    %
\definecolor{pblueL}{HTML}{9AB5D1}   %
\definecolor{pgreenD}{HTML}{1D5640}  %
\definecolor{pgreen}{HTML}{3B7D5E}   %
\definecolor{pgreenL}{HTML}{A3C9B5}  %
\definecolor{pteal}{HTML}{2F7F86}   %
\definecolor{polive}{HTML}{6B7A2F}  %
\definecolor{pred}{HTML}{8E3B33}    %
\definecolor{pgrey}{HTML}{6B7280}    %
\definecolor{pgreyL}{HTML}{D5D9DE}   %
\definecolor{ptext}{HTML}{1F2933}    %

\newcommand{\ind}{\mathbf{1}}

\makeatletter
\newcommand{\para}[1]{\ifvmode\if@nobreak\else\vspace{0.1cm plus 2pt}\fi\else\unskip\vspace{0.1cm plus 2pt}\\\fi\noindent\textbf{#1}\ }
\makeatother

\newcommand{\method}{\mbox{\textup{ASC}}}  %

\colorlet{oursrow}{pblueL!35}
\makeatletter
\newcommand{\shadeL}{\gdef\CT@cell@color{\CT@color{oursrow}\@tempdimb\z@\global\let\CT@cell@color\relax}}
\newcommand{\shadeR}{\gdef\CT@cell@color{\CT@color{oursrow}\@tempdimc\z@\global\let\CT@cell@color\relax}}
\newcommand{\shadeLR}{\gdef\CT@cell@color{\CT@color{oursrow}\@tempdimb\z@\@tempdimc\z@\global\let\CT@cell@color\relax}}
\makeatother

\newcommand{\shortcut}[1]{\textcolor{black!38}{#1}}

\colorlet{shortcutcell}{pgreenL!55}
\newcommand{\extag}[1]{\textsubscript{\textsf{\scriptsize #1}}}
\newcommand{\exC}[1]{\textcolor{pgreenD}{#1}\extag{C}}
\newcommand{\exS}[1]{\textcolor{pblueD}{#1}\extag{A}}
\newcommand{\exU}[1]{\textcolor{pred}{#1}\extag{U}}
\newcommand{\exV}[1]{\textcolor{pgrey}{\textit{#1}}\extag{V}}
\usepackage{soul}
\colorlet{exsrccol}{pgreenL!50}
\newcommand{\exsrc}[1]{{\sethlcolor{exsrccol}\hl{#1}}}
\newcommand{\exlabel}[1]{\noindent\textbf{#1}\par\smallskip}
\usepackage{tcolorbox}
\tcbuselibrary{skins,breakable}
\newtcolorbox{exbox}[2]{enhanced, colback=white, colframe=black!55, boxrule=0.5pt, arc=1.5pt,
  fonttitle=\bfseries\footnotesize, coltitle=black, colbacktitle=black!7, titlerule=0.4pt, toptitle=1pt,
  bottomtitle=1pt, title={#1}, left=4pt, right=4pt, top=3pt, bottom=3pt,
  sidebyside, sidebyside align=top seam, sidebyside gap=10pt, lefthand width=#2\linewidth,
  segmentation style={black!25}, fontupper=\scriptsize\color{black!85}, fontlower=\footnotesize}
\newtcolorbox{oursbox}{enhanced, breakable, colback=oursrow, frame hidden, sharp corners, boxsep=0pt,
  left=2pt, right=2pt, top=1.5pt, bottom=1.5pt, before skip=1pt, after skip=1pt}
\newtcolorbox{promptbox}[2][]{enhanced, colback=white, colframe=black, boxrule=0.5pt, arc=1.5pt,
  left=5pt, right=5pt, top=3pt, bottom=4pt, fonttitle=\bfseries\footnotesize, coltitle=black,
  colbacktitle=black!7, titlerule=0.4pt, toptitle=1pt, bottomtitle=1pt, fontupper=\footnotesize, #1}
\newcommand{\slot}[1]{\textit{\{#1\}}}

%% file: sections/00_abstract.tex
\begin{abstract}
In search engines and assistants, long-form RAG answers can mislead readers with claims their sources do not support.
Faithfulness rewards, the share of supported sentences, train models toward fewer such claims.
We show that these rewards are hacked by copying.
A sentence copied from the passages is supported almost by construction, so trained models learn to copy rather than ground what they write.
Common defenses only move the hack: a KL or copy penalty leaves copying in place or trades it for shorter, emptier answers and refusals.
GRPO, the standard RL algorithm, makes it worse by giving every sentence its answer's advantage, reinforcing copied and grounded sentences alike.
We propose Authored-Sentence Credit (\textbf{\method{}})\footnote{Our code is available at \url{https://anonymous.4open.science/r/asc14}.}, which fixes both by moving GRPO's reward and advantage from the answer to the sentence level.
The reward credits only sentences that are authored rather than copied, say more than the question, and are supported by the passages.
The advantage compares each sentence with every sentence sampled for the question, so one that earns no credit ranks last and is never reinforced.
Across four models and two in-domain, three held-out corpora, only \method{} lowers the error rate of authored sentences and writes more supported ones in every setting, by 38\% and 46\% at the median under held-out evaluators, giving readers more grounded information.
It takes none of these shortcuts, and two LLM judges comparing whole answers prefer no baseline to the base model but rate \method{}'s answers on par with it.
We conclude that a faithfulness reward should credit what a model writes, not what it copies.
\end{abstract}

%% file: sections/01_intro.tex
\section{Introduction}
\label{sec:intro}

Retrieval-augmented generation \citep[RAG;][]{lewis_rag_2020} now answers open questions in search engines and assistants, yet its long-form answers often state claims their sources do not support, in web search \citep{liu_verifiability_2023} as in law and medicine \citep{magesh_legal_2025,wu_sourcecheckup_2025}.
Faithfulness, the share of an answer's sentences a checker finds supported by the passages \citep{gao_alce_2023,min_factscore_2023,tang_minicheck_2024}, is now a training signal against such claims, as a reward for reinforcement learning or to rank preference pairs \citep{huang_citations_2024,chen_factuality_2026,tian_facttune_2024}.

\begin{figure}[t]
\centering
\includegraphics[width=\columnwidth]{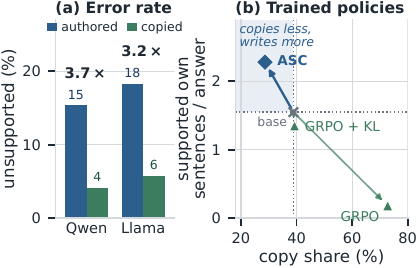}
\caption{\textbf{A faithfulness reward can be hacked by copying} (ASQA).
\textbf{(a)}~Under the reward's checker, the sentences a model writes itself are unsupported 3.2 to 3.7 times as often as those it copies.
\textbf{(b)}~Qwen3.5-9B after training: GRPO copies and nearly stops writing, GRPO~+~KL writes less, and \method{} copies less and writes more supported sentences of its own (held-out checker).}
\label{fig:teaser}
\end{figure}

The reward has a cheaper way up.
A sentence copied from the passages is supported almost by construction: on ASQA, the sentences Qwen3.5-9B writes itself are unsupported 3.7 times as often as those it copies, and those of Llama-3.1-8B 3.2 times (Figure~\ref{fig:teaser}a).
The reward therefore rises with every sentence a model copies instead of writing.

Trained policies exploit this gap, a form of reward hacking \citep{skalse_hacking_2022}.
GRPO \citep{shao_grpo_2024} raises the copy share of Qwen3.5-9B from 38.8 to 72.8 (Figure~\ref{fig:teaser}b) and shrinks its answers from 74 to 14 words: the reward climbs because the model stops writing, not because it grounds better.
Common defenses only move the hack between three \emph{shortcuts} (copying, writing less and writing emptier sentences): on the same model, a KL penalty holds copying at the base level, but refusals more than double and answers grow shorter and emptier, and a copy penalty makes it decline most questions.
GRPO itself makes it worse: it gives every token of an answer the same advantage, its update weight, so a copied or empty sentence in a high-reward answer is reinforced and a grounded one in a low-reward answer pushed down.

We propose Authored-Sentence Credit (\textbf{\method{}}), to our knowledge the first faithfulness reward that pays only for what the model writes itself (Figure~\ref{fig:method}).
It changes GRPO in two places and adds no model.
Its reward credits a sentence only if it is not copied, says more than the question (not \emph{vacuous}), and is supported by the passages.
Its advantage scores each sentence against all sentences sampled for the question, so each sentence is judged by its own credit, not its answer's.
Together, they guarantee that a copied, vacuous or unsupported sentence ranks last in its group and is never reinforced.

We train four models from two families on two corpora and test them on five, three held out.
Every baseline, even a sentence-level reward, lowers the error rate of authored sentences at the median, yet takes a shortcut in nearly every setting.
\method{} alone lowers it while writing more, with no shortcut in any setting: at the median, it cuts that error rate by 38\% under an independent judge and writes 46\% more supported sentences of its own, giving readers more grounded information; human annotators confirm the drop without the judge.
A higher reward alone does not give better answers: every baseline raises the reward, yet two LLM judges comparing whole answers prefer none to the base model but rate \method{} on par with it and above every baseline.

\noindent In summary, our contributions are as follows.
\begin{itemize}[leftmargin=*,itemsep=1pt,topsep=2pt,parsep=0pt]
  \item \textbf{Analysis.} We show that faithfulness rewards pay for copying, that KL and copy penalties only move the hack to other shortcuts, and that GRPO's answer-level advantage barely tells copied from grounded sentences.
  \item \textbf{Method.} We propose \method{}, to our knowledge the first faithfulness reward that credits only authored sentences, and prove that no uncredited sentence is ever reinforced.
  \item \textbf{Results.} Only \method{} lowers the authored error rate and writes more supported sentences in all settings, with no shortcut and no quality loss.
\end{itemize}

%% file: sections/02_related_work.tex
\section{Related Work}
\label{sec:related}

\para{Faithfulness metrics favor copying.}
Long-form answers are scored statement by statement against their sources \citep{gao_alce_2023,min_factscore_2023}, and small checkers make this scoring cheap enough to serve as a reward \citep{tang_minicheck_2024,zha_alignscore_2023}.
In summarization, such metrics favor extractive output, so faithfulness gains often come from copying \citep{ladhak_effective_2022,dreyer_tradeoff_2023}, and an entailment reward can trade informativeness for consistency \citep{roit_entailment_2023}.
We study this bias as a reward hack in long-form RAG, where copying has a cost for readers: in generative search, more precise citations come with more copying from the sources and lower perceived utility \citep{liu_verifiability_2023}.
Trivially true claims also inflate factuality scores \citep{jiang_core_2025}; our vacuity test removes them by checking each sentence against the question alone, as context-contrast methods do \citep{shi_cad_2024,cohenwang_contextcite_2024}.

\para{Grounded generation and reward hacking.}
Faithfulness and factuality objectives train models through preference optimization \citep{tian_facttune_2024,song_trustalign_2024}, reinforcement learning \citep{si_canoe_2025,huang_citations_2024}, self-reflection \citep{asai_selfrag_2024} and citation tuning \citep{ye_agree_2024}, and one method even trains RAG models to copy more so that they hallucinate less \citep{long_copypaste_2025}.
Optimizing such a proxy invites reward hacking \citep{skalse_hacking_2022,gao_overoptimization_2023}, best studied for length \citep{singhal_length_2023}; for factuality, precision-only metrics and rewards favor short, less detailed answers \citep{wei_safe_2024,chen_factuality_2026}.
A KL penalty does not prevent hacking under heavy-tailed reward error \citep{kwa_goodhart_2024}, and optimization amplifies a reward's biases \citep{shapira_sycophancy_2026,moya_tie_2026}.
We show that faithfulness rewards are hacked by copying, and that common defenses only move the hack.

\para{Fine-grained credit assignment.}
Dense signals score each sentence \citep{wu_finegrained_2023} or reasoning step \citep{lightman_verify_2024}.
GRPO normalizes rewards within a group of answers \citep{shao_grpo_2024}, and GiGPO within groups of steps \citep{feng_gigpo_2025}.
The closest rewards score facts or sentences for closed-book correctness \citep{wen_rlfh_2025,gu_maskdpo_2025} or for support in RAG \citep{huang_citations_2024}, but none asks whether a sentence was copied.
\method{} withholds credit from copied and vacuous sentences and compares sentences, not answers, so these are never reinforced.

%% file: sections/03_setup.tex
\section{Problem Formulation and Setup}
\label{sec:setup}

\para{Task and reward.}
A prompt consists of a question \(q\) and retrieved passages \(P\), and a policy \(\pi_\theta\) writes a long-form answer \(y=(y_1,\dots,y_n)\) of \(n\) sentences.
A checker \(g\) (e.g., Bespoke-MiniCheck-7B; \citealp{tang_minicheck_2024,bespoke_minicheck_2024}) returns \(g(y_s,P)\in[0,1]\), its probability that \(P\) supports \(y_s\), and the sentence counts as supported when \(g(y_s,P)\ge\kappa\).
The faithfulness reward that training maximizes and evaluations report is the share of supported sentences,
\begin{equation}\textstyle
R(y)\;=\;\frac{1}{n}\sum_{s=1}^{n}\ind\!\left[g(y_s,P)\ge\kappa\right].
\label{eq:reward}
\end{equation}

\noindent\textbf{Training.}
GRPO \citep{shao_grpo_2024} samples \(G\) answers to each prompt and gives every token of answer \(i\) the same advantage, its reward standardized over the group, \(A^{(i)}=(R(y^{(i)})-\mu)/(\sigma+\epsilon)\) with \(\mu\), \(\sigma\) the group's mean and standard deviation; DPO \citep{rafailov_dpo_2023} instead learns offline from pairs of answers, preferring the one with higher \(R\).

\para{Copied, authored and vacuous sentences.}
The copy share \(c(y_s)\) of a sentence is the fraction of its words inside verbatim matches of at least \(m\) words with \(P\), and \(c(y)\) is the same fraction over an answer.
A sentence is \emph{copied} if \(c(y_s)\ge\tau\) and \emph{authored} otherwise; \(\rho(y)\) is the share of copied sentences in an answer.
A sentence is \emph{vacuous} if it passes the same support test with the question in place of the passages, \(\mbox{\(g(y_s,q)\ge\kappa\)}\): it needs no evidence to count as supported, like ``The song was released.'' for a question about a song.
Only authored, non-vacuous sentences show whether a model can support what it writes itself; we call one \emph{grounded} when \(P\) supports it.

\para{Goal and shortcuts.}
Training should make a model's own sentences more reliable: it should lower the \emph{authored error rate} \(e_a(\pi)\), the share of authored sentences that the passages do not support, which is 2--40\% for the base model \(\pi_0\).
It should do so without writing less: the copy share \(c\), vacuous share \(v\) and refusal rate \(d\) (the share of answers that decline) should not rise above those of \(\pi_0\), nor the answer length \(\ell\) fall below it.
\(R\) does not encode this goal and can be hacked through three \emph{shortcuts}, copying more, writing less and writing emptier sentences, with refusal as the extreme; each raises \(R\), and writing less can even lower \(e_a\) by dropping risky sentences, with no authored sentence becoming more reliable.
Section~\ref{sec:analysis} shows that training on \(R\) takes these shortcuts, and Section~\ref{sec:method} introduces a signal that improves on \(\pi_0\) without them.

\para{Metrics.}
Two lead metrics measure the goal; the others detect shortcuts and check quality (Table~\ref{tab:app-metrics}).
\begin{itemize}[leftmargin=*,itemsep=1pt,topsep=2pt,parsep=0pt]
  \item \textbf{Reliability (Err).} \(e_a\), judged with the passages by Gemma-4-31B \citep{gemma_gemma4_2026}, which comes from a different family than the reward's checker and is never used in training.
  \item \textbf{Content (Supp.).} Supported authored sentences per answer under the held-out MiniCheck-Flan-T5-Large \citep{tang_minicheck_2024}; unlike Err, it falls when a model simply writes less.
  \item \textbf{Shortcuts (Copy, Vac., Ref., Words).} \(c\), \(v\), \(d\) and \(\ell\); a policy \emph{takes a shortcut} when \(c\), \(v\) or \(d\) exceeds that of \(\pi_0\) by a set margin, or \(\ell\) falls below a set fraction of it (values in Section~\ref{sec:exp-setup}).
  \item \textbf{Checks.} AlignScore \citep{zha_alignscore_2023} as a second held-out checker, ASQA exact match (EM; \citealp{gao_alce_2023}), and two LLM judges of whole answers (Section~\ref{sec:main}).
\end{itemize}

\para{Models and data.}
We train four instruction models from two families, Qwen3.5-9B, Qwen3.5-4B \citep{qwen_qwen35_2026}, Llama-3.1-8B and Llama-3.2-3B \citep{grattafiori_llama3_2024}, and evaluate each on five corpora, 20 \emph{settings} of model and corpus in all.
Training pools ASQA \citep{stelmakh_asqa_2022} with ALCE's top five passages \citep{gao_alce_2023} and the question-answering part of RAGTruth \citep{niu_ragtruth_2024}; we evaluate in domain on their evaluation sets.
Three further corpora, ClapNQ \citep{rosenthal_clapnq_2025}, ELI5 \citep{fan_eli5_2019} with ALCE's passages and ExpertQA \citep{malaviya_expertqa_2024}, are held out entirely, so a gain on them cannot come from fitting the training data.
On ELI5 and ExpertQA, where base models err most and copy least, a good answer has to be written, not assembled from the passages.

%% file: sections/04_analysis.tex
\section{Why Faithfulness Rewards Are Hacked}
\label{sec:analysis}

We ask what \(R\) pays for in base-model answers, what a policy trained on it learns, and what GRPO's advantage sees inside an answer; each finding yields one requirement that Section~\ref{sec:method} meets.

\begin{figure}[!h]
\centering
\includegraphics[width=\columnwidth]{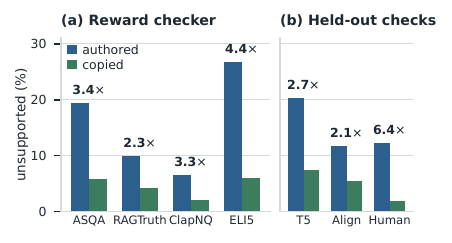}
\caption{Unsupported rate of authored and copied sentences, with their ratio. \textbf{(a)}~Base-model samples under the reward's checker (mean over models). \textbf{(b)}~The same samples under MiniCheck-Flan-T5 (T5) and AlignScore (Align), mean over models and corpora, and human labels on RAGTruth's responses.}
\label{fig:finding1}
\end{figure}

\para{Finding 1: copied sentences are rarely unsupported, so the reward pays for copying.}
We sample eight answers per question from each base model on ASQA, RAGTruth, ClapNQ and ELI5.
Under the reward's checker, authored sentences are unsupported 2.3 to 4.4 times as often as copied ones on average (Figure~\ref{fig:finding1}a), and more often for every model and corpus.
The gap is not an artifact of the checker: two held-out checkers find it almost everywhere, and so do human labels on RAGTruth's responses, 12.2\% against 1.9\% (Figure~\ref{fig:finding1}b).
With \(s_c\) and \(s_a\) the support rates of copied and authored sentences, the reward of Eq.~\ref{eq:reward} decomposes as
{\setlength{\abovedisplayskip}{6pt plus 1pt}\setlength{\belowdisplayskip}{6pt plus 1pt}%
\begin{equation}\textstyle
R \;=\; s_a \;+\; \rho\,(s_c - s_a),
\label{eq:line}
\end{equation}}%
a line in \(\rho\) with slope \(s_c-s_a>0\): at fixed authored reliability, every sentence copied instead of written raises \(R\) (proof in Appendix~\ref{app:line}).
Sampled answers agree: among answers to the same question, those that copy more earn more reward for every model and corpus (slopes in Appendix~\ref{app:finding1}).
Yet a copied sentence inherits its source's reliability, so its credit says nothing of the model's own grounding.
This gives the first requirement, \textbf{(R1)}: a copied sentence earns nothing, however well supported.

\begin{table}[!h]
\centering
\footnotesize
\setlength{\tabcolsep}{1.8pt}
\input{tables/tab_finding2}
\caption{Qwen3.5-9B trained to raise \(R\), on ASQA; metrics of Section~\ref{sec:setup}, rates in \%. \textbf{Bold}: a shortcut. Llama-3.1-8B and training curves: Appendix~\ref{app:finding2}.}
\label{tab:finding2}
\end{table}

\begin{figure}[!h]
\centering
\includegraphics[width=\columnwidth]{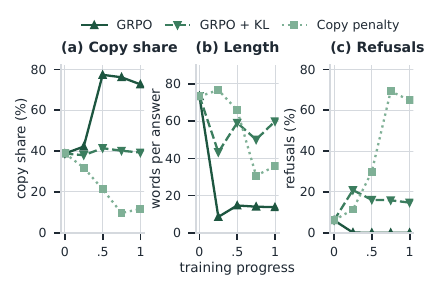}
\caption{Qwen3.5-9B on ASQA over training (0: base model): GRPO copies and nearly stops writing, and each defense swaps one shortcut for another.}
\label{fig:finding2}
\end{figure}

\para{Finding 2: training takes the shortcuts, and common defenses only move between them.}
We train Qwen3.5-9B to raise \(R\) with GRPO, alone or with one of two defenses, a KL penalty \citep{ziegler_finetuning_2019,ouyang_instructgpt_2022} or a copy penalty, and with DPO (Table~\ref{tab:finding2}).
Plain GRPO nearly saturates the reward (98.9) through the cheapest shortcuts, and the answers get worse: copy share nearly doubles within half of training, answers shrink from 74 to 14 words (Figure~\ref{fig:finding2}), and exact match falls from 45.5 to 27.6.
Each defense closes one shortcut and opens another: the KL penalty holds copying at the base level, but answers grow shorter and emptier and refusals rise, and the copy penalty cuts copying, but its vacuous share rises to 67.4\% and it declines 64.5\% of questions.
DPO copies too, even on pairs whose answers copy nearly equally: preferring better-supported answers alone pulls it toward the passages.
The two penalties take a shortcut on nine in ten model--corpus pairs (grey in Table~\ref{tab:main}).
\(R\) pays for any sentence the checker cannot reject, so penalizing one quantity only reroutes the optimizer.
This gives the second requirement, \textbf{(R2)}: a vacuous sentence earns nothing, and dropping sentences or refusing never pays.

\begin{figure*}[t]
\centering
\includegraphics[width=\textwidth]{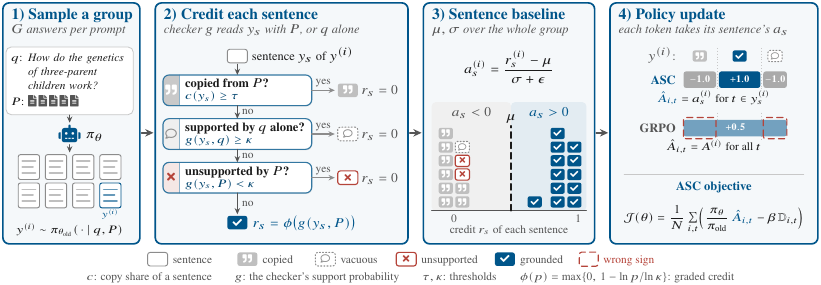}
\caption{\textbf{One \method{} step in GRPO.}
1)~The policy samples \(G\) answers to a question \(q\) with passages \(P\).
2)~A sentence earns credit only if it is authored, not vacuous and supported by \(P\).
3)~Each credit is compared with those of all sentences in the group, across answers.
4)~Every token takes its sentence's advantage; the rest of GRPO is unchanged.
Example (Qwen3.5-9B, ELI5): answer \(y^{(i)}\) earns the full reward although one sentence is copied and one vacuous; GRPO pushes all three up (\(+0.5\)), \method{} only the grounded one (\(+1.0\); the others \(-1.0\)).}
\label{fig:method}
\end{figure*}

\para{Finding 3: GRPO's advantage barely tells a copied sentence from a grounded one.}
GRPO's advantage belongs to the whole answer, so a sentence is rewarded for the answer it sits in: a copied or vacuous sentence in a high-reward answer is pushed up, and a grounded one in a low-reward answer is pushed down.
Copying therefore barely changes a sentence's chance of being reinforced: on Finding~1's samples, of the sentences GRPO moves, it pushes up 72.0\% of the copied ones and 75.5\% of the grounded ones, and even 58.8\% of the vacuous ones; grounded ones are pushed up more often everywhere, but by at most 17 points (Appendix~\ref{app:finding3}).
An answer-level advantage thus cannot steer the policy away from copying: it reinforces whatever the high-reward answers contain.
This gives the third requirement, \textbf{(R3)}: sentences are compared with sentences, so that each advantage reflects the sentence's own reliability.

%% file: tables/tab_finding2.tex
\begin{tabular}{@{}lrrrrrr@{}}
\toprule
\textbf{Training} & \textbf{Reward} & \textbf{Copy} & \textbf{Words} & \textbf{Vac.} & \textbf{Ref.} & \textbf{EM} \\
\midrule
Base model & 87.1 & 38.8 & 73.5 & 21.5 & 6.1 & 45.5 \\
GRPO & 98.9 & \textbf{72.8} & \textbf{13.7} & 18.1 & 0.0 & 27.6 \\
GRPO + KL & 92.3 & 39.2 & \textbf{60.0} & \textbf{28.5} & \textbf{14.7} & 41.8 \\
Copy penalty & 92.0 & 11.8 & \textbf{36.9} & \textbf{67.4} & \textbf{64.5} & 20.2 \\
DPO & 94.6 & \textbf{53.4} & \textbf{55.7} & 16.9 & 1.7 & 41.8 \\
DPO, copy-matched & 90.4 & \textbf{44.9} & \textbf{59.7} & 21.5 & 6.1 & 43.8 \\
\bottomrule
\end{tabular}

%% file: sections/05_method.tex
\section{\method{}: Authored-Sentence Credit}
\label{sec:method}

\method{} meets (R1)--(R3) with two changes to GRPO, which sentences earn credit (Section~\ref{sec:credit}) and what each credit is compared with (Section~\ref{sec:baseline}), and leaves the rest unchanged (Section~\ref{sec:objective}; Figure~\ref{fig:method}).

\subsection{Which sentences earn credit}
\label{sec:credit}

For a sampled answer \(y\) with sentences \(y_1,\dots,y_n\), the credit of \(y_s\) is the product of three tests,
\begin{equation}\textstyle
\mskip-8mu r_s\!=\!\underbrace{\ind[c(y_s)\!<\!\tau]}_{\text{copy}}\underbrace{\ind[g(y_s,q)\!<\!\kappa]}_{\text{vacuity}}\underbrace{\phi(g(y_s,P))}_{\text{support}}.
\label{eq:credit}
\end{equation}
The copy test meets (R1): a copied sentence owes its support to repeating \(P\), not to the policy's grounding, so the test gives it what an unsupported sentence earns.
The vacuity test meets the first half of (R2): a sentence that the checker accepts from the question alone needs no evidence, so its support is no evidence of grounding either.
The support test credits the remaining sentences with the graded credit
\begin{equation}\textstyle
\phi(p)\;=\;\max\Big\{0,\;1-\frac{\ln p}{\ln\kappa}\Big\}.
\label{eq:phi}
\end{equation}

\para{Properties of the credit.}
\(\phi\) is zero for \(p\le\kappa\) and one at \(p=1\), so the credited sentences are exactly those above the checker's threshold.
Writing \(h(p)=-\ln p\) for the checker's surprisal at ``supported'', \(\phi(p)=1-h(p)/h(\kappa)\) on \((\kappa,1]\): \(\phi\) is the only credit that is affine in surprisal, zero at the threshold and one at certainty, and it is strictly increasing and concave in \(p\), rising fastest just above the threshold (Appendix~\ref{app:credit}).
The binary credit \(\ind[p\ge\kappa]\) pays a sentence just above the threshold as much as one the checker is certain of; the graded credit keeps the checker's confidence and pulls the policy toward sentences the checker is sure of.
Every credit lies in \([0,1]\), and a sentence that fails any test earns the minimum, \(r_s=0\).

\subsection{A sentence baseline}
\label{sec:baseline}

Where GRPO standardizes answers (Section~\ref{sec:setup}), \method{} standardizes sentences.
For a group of \(G\) answers with \(M\) sentences in total, let \(\mu\) and \(\sigma\) be the mean and standard deviation of the credits of all \(M\) sentences.
Sentence \(s\) of answer \(i\) receives the advantage
{\setlength{\abovedisplayskip}{3pt plus 1pt}\setlength{\belowdisplayskip}{4pt plus 1pt}%
\setlength{\abovedisplayshortskip}{3pt plus 1pt}\setlength{\belowdisplayshortskip}{4pt plus 1pt}%
\begin{equation}\textstyle
a^{(i)}_s \;=\; \frac{r^{(i)}_s-\mu}{\sigma+\epsilon}, \qquad \epsilon>0 .
\label{eq:adv}
\end{equation}}%
Proposition~\ref{prop:adv} shows that the credit and the baseline together meet (R1)--(R3): whatever the answers contain, a sentence failing a test ranks strictly below every sentence that earns credit.

\begin{proposition}[No shortcut earns advantage]
\label{prop:adv}
For any group and the advantage of Eq.~\ref{eq:adv}:
\begin{enumerate}[label=(\roman*),nosep,leftmargin=*]
\item \emph{Credit alone sets the order}: \(a^{(i)}_s< a^{(j)}_t\) if and only if \(r^{(i)}_s< r^{(j)}_t\), for sentences of the same or of different answers (R3).
\item \emph{Shortcuts rank last}: a copied, vacuous or unsupported sentence receives the lowest advantage of its group, \(-\mu/(\sigma+\epsilon)\), negative if any sentence earns credit (R1, R2).
\item \emph{Average sentences earn nothing}: advantages sum to zero over the group, and a sentence of average credit receives zero.
\item \emph{Length-free advantages}: given the group's \(\mu\) and \(\sigma\), a sentence's advantage depends on neither its own length nor its answer's (R2).
\end{enumerate}
\end{proposition}

All four parts follow from two facts: Eq.~\ref{eq:adv} applies one strictly increasing affine map to every credit of the group, and a sentence that fails a test has the minimum credit.
GRPO's answer-level advantage breaks (i) and (ii), as Finding~3 measured, and (iv) as well (proofs in Appendix~\ref{app:prop}).

\subsection{Training objective}
\label{sec:objective}

Every token \(t\) of sentence \(s\) in answer \(i\) takes that sentence's advantage, \(\hat A_{i,t}=a^{(i)}_s\), in place of GRPO's \(A^{(i)}\); tokens outside every sentence, the whitespace between sentences and the end-of-sequence token, take zero.
The objective is otherwise GRPO's (Algorithm~\ref{alg:method}, Appendix~\ref{app:setup}):
\begin{equation}\textstyle
\mathcal{J}(\theta)=\frac{1}{N}\sum_{i,t}\Big(w_{i,t}(\theta)\,\hat A_{i,t}-\beta\,\mathbb{D}_{i,t}\Big),
\label{eq:objective}
\end{equation}
where \(w_{i,t}(\theta)=\pi_\theta(y^{(i)}_t\mid q,P,y^{(i)}_{<t})/\pi_{\theta_{\mathrm{old}}}(y^{(i)}_t\mid q,P,y^{(i)}_{<t})\) is the importance ratio, \(\mathbb{D}_{i,t}\) GRPO's per-token KL estimate \citep{shao_grpo_2024} toward the reference \(\pi_{\mathrm{ref}}\), the base model, and \(N\) the number of answer tokens in the batch \citep{yu_dapo_2025}.

\para{Gradient form.}
Each batch is used for one update, so the gradient is taken where \(w_{i,t}=1\), and the log-probabilities of a sentence's tokens add up to its log-likelihood.
Up to the KL term, the gradient of Eq.~\ref{eq:objective} is a policy gradient over sentences,
\begin{equation}\textstyle
\nabla\mathcal{J}=\frac{1}{N}\sum_{i,s}a^{(i)}_s\,\nabla\ln\pi_\theta\big(y^{(i)}_s\mid y^{(i)}_{<s}\big),
\label{eq:grad}
\end{equation}
with \(\nabla=\nabla_\theta\), \(y^{(i)}_{<s}\) the text before sentence \(s\) and the prompt \((q,P)\) left implicit (Appendix~\ref{app:grad}).
GRPO's gradient weights the log-likelihood of the whole answer by \(A^{(i)}\), so every sentence moves with its answer; \method{} moves each sentence by its own advantage, whatever the rest of its answer contains, which is the formal content of (R3).

\para{Relation to GRPO and to sentence rewards.}
If every answer had a single sentence, Eq.~\ref{eq:adv} would reduce to GRPO's own standardization with \(r\) as the reward, so \method{} is GRPO with a finer unit of comparison.
Without the first two tests and with the binary credit, \method{} becomes a sentence-level support reward \citep{wu_finegrained_2023,huang_citations_2024}, the sentence-reward baseline of Section~\ref{sec:exp-setup}.
The only added checker cost is one question-only call per authored sentence.
\method{} needs no critic, no reference answers and no penalty weight.

%% file: sections/06_experiments.tex
\section{Experiments}
\label{sec:experiments}

\subsection{Setup}
\label{sec:exp-setup}

\para{Training.}
All policies are trained with GRPO at full weight on the pooled ASQA and RAGTruth training questions, with \(G=8\) answers per question, a KL coefficient \(\beta=0.04\), one zero-shot prompt and the same hyper-parameters for every method (Appendix~\ref{app:setup}).
Every method uses the same reward checker, Bespoke-MiniCheck-7B \citep{tang_minicheck_2024,bespoke_minicheck_2024}, and the same thresholds, \(\kappa=\tau=\tfrac12\) and \(m=5\) words, so the methods differ only in the training signal.

\para{Baselines.}
We compare \method{} with three baselines.
\textbf{GRPO + KL} trains on the faithfulness reward \(R\) of Eq.~\ref{eq:reward}, the standard recipe.
\textbf{Copy penalty} trains on \(R-\lambda\,c(y)\), the direct fix for the copy shortcut, with \(\lambda\) set to the support gap \(s_c-s_a\) of the model's own samples on the training questions, so that by Eq.~\ref{eq:line} a copied sentence earns no more than an authored one.
\textbf{Sentence reward}, a sentence-level support reward in the style of fine-grained RLHF \citep{wu_finegrained_2023}, is \method{} without its first two tests and graded credit, so it keeps the sentence baseline.

\para{Evaluation.}
We decode greedily and evaluate on all five corpora with the metrics of Section~\ref{sec:setup}.
A policy takes a shortcut on a corpus when its copy share, vacuous share or refusal rate exceeds the base model's by more than 3 points, or its answers fall below 85\% of the base model's words.
Rows average the training seeds where a method has several (\method{} on the Qwen models, binary credit) and are single runs otherwise; Appendix~\ref{app:full} gives per-corpus results, every run and bootstrap significance tests, and Appendix~\ref{app:examples} example answers.

\begin{table*}[t]
\centering
\footnotesize
\setlength{\tabcolsep}{2.6pt}
\renewcommand{\arraystretch}{0.9}
\input{tables/tab_main}
\caption{Main results (0--100). Err: authored error rate on two in-domain (ASQA, RAGTruth) and three held-out corpora, with its mean; answer metrics: means over the five corpora; EM: exact match on ASQA. Grey Err: a shortcut on that corpus (Section~\ref{sec:exp-setup}), not ranked; grey mean: a shortcut on at least one corpus. Bold: lowest remaining Err and most supported authored sentences per answer (Supp.); underlined: second-lowest Err.}
\label{tab:main}
\end{table*}

\subsection{Main results}
\label{sec:main}

\begin{figure}[t]
\centering
\includegraphics[width=\columnwidth]{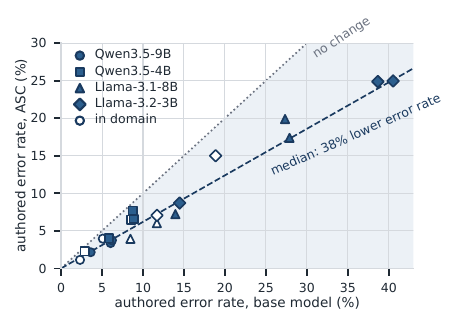}
\caption{Authored error rate of the base model against \method{}, one point per model and corpus (hollow: in domain); all lie below the diagonal (dashed: median).}
\label{fig:gain}
\end{figure}

\para{Only \method{} lowers the error rate while writing more.}
Every method lowers the authored error rate at the median, \method{} the most (38\%, against 31 to 36\%), but a model can also lower it by writing less, so Err alone does not separate them (Table~\ref{tab:main}).
What separates them is how much grounded content the model still writes (Supp.): \method{} writes the most supported authored sentences on every model, 46\% more than the base model at the median and more in every setting, whereas GRPO + KL writes fewer than the base model on three of the four models.
The lower rate is not a dilution by the extra sentences: unsupported authored sentences per answer fall too, in 16 of the 20 settings.
Figure~\ref{fig:gain} shows that the cut holds at every error level, including the hardest settings, where the base model errs on about 40\% of its sentences.
It is even larger on the held-out corpora, never seen in training, than in domain (median 38\% against 32\%).

\para{\method{} takes no shortcut; the baselines take one almost everywhere.}
In every setting \method{} copies less than the base model, writes fewer vacuous sentences and refuses less often, so none of its Err values is grey in Table~\ref{tab:main}, under any margin from 1 to 10 points.
The baselines buy their lower error rates with a shortcut, and closing one opens another, as in Finding~2: GRPO + KL declines up to 43.9\% of questions on the Qwen models, against about 5\% before training, and copies more on the Llama models; the copy penalty cuts copying but declines 79.7\% of questions on Qwen3.5-9B; and the sentence reward copies more than the base model in every setting, since scoring sentences one by one does not close the copy shortcut.
Where a baseline takes no shortcut (5 of 60), its error rate is never significantly below \method{}'s, and \method{} writes significantly more supported authored sentences in all five, so no baseline matches it even there.

\para{\method{} gains faithfulness at no cost in quality.}
Two LLM judges, Gemma-4-31B and Mistral-Small-3.2-24B \citep{mistral_small32_2025} from a third model family, pick, in both orders, the answer that serves a reader better, under a rubric fixed on base-model probes before any trained answer was judged, to favor neither copied text nor padded answers (Appendix~\ref{app:quality}).
For every baseline the faithfulness gains cost quality (Table~\ref{tab:quality}): each raises the reward over the base model, yet neither judge prefers any of them to it.
\method{} is the only method that pays nothing: both judges rate its answers on par with the base model's while they carry fewer errors and more supported sentences, and prefer them to those of every baseline.
Blind human annotators, with the same rubric, prefer \method{}'s answer to the base model's more often than the reverse (22 against 12 of 75 pairs) and GRPO + KL's (28 against 14; Appendix~\ref{app:human}).
Exact match on ASQA stays within one point of the base model for \method{}, while the copy penalty loses up to 25 points (Table~\ref{tab:main}).

\begin{table}[!h]
\centering
\footnotesize
\setlength{\tabcolsep}{3pt}
\begin{adjustbox}{max width=\columnwidth}
\input{tables/tab_quality}
\end{adjustbox}
\caption{Answer quality over the 20 settings: training reward, win rate (\%, a tie counts half) against the base model and of \method{} against each method under two judges, and coverage (\%) of the reference answer; the base model's reward is 83.5 and its coverage 45.0.}
\label{tab:quality}
\end{table}

\subsection{Each part of \method{} contributes}
\label{sec:ablation}

\begin{table}[!h]
\centering
\footnotesize
\setlength{\tabcolsep}{2.2pt}
\begin{adjustbox}{max width=\columnwidth}
\input{tables/tab_ablation}
\end{adjustbox}
\caption{Ablations on Qwen3.5-9B, one part of \method{} changed at a time: Err on each corpus, and Supp., copy and vacuous share averaged over the five corpora (grey Err: a shortcut, as in Table~\ref{tab:main}). Binary credit: \(\ind[p\ge\kappa]\) in place of \(\phi\); the sentence reward also drops both tests.}
\label{tab:ablation}
\end{table}

Table~\ref{tab:ablation} changes one part of \method{} at a time on Qwen3.5-9B: removing any part brings back a shortcut or raises the error rate on most corpora.
Without the copy test, copying returns (copy share 33.3 against 22.2, above the base model's 28.1), so the variant's lower error rates are grey, and they come with fewer supported authored sentences (2.6 against 3.0 per answer).
Without the vacuity test, vacuous sentences return (19.1\% against 13.3\%) and the error rate on held-out ExpertQA doubles, from 3.7 to 7.5.
With GRPO's whole-answer advantage in place of the sentence baseline, the policy writes more (Supp.~3.3) but less reliably: the error rate rises on four of the five corpora and the vacuous share to 18.3\%, as Finding~3 predicts.
With binary credit \(\ind[p\ge\kappa]\) in place of the graded credit \(\phi\) of Eq.~\ref{eq:phi}, the error rate rises on all five corpora at similar copy and vacuous shares: the graded credit favors sentences the checker is sure of, which the judge accepts more often.
The sentence reward, which drops both tests and the graded credit, copies more than the base model and writes nearly twice as many vacuous sentences as \method{}.

\subsection{Robustness}
\label{sec:analysis-exp}

The gain holds across runs, evaluators and training, and it does not game the copy test.
\begin{itemize}[leftmargin=*,itemsep=1pt,topsep=2pt,parsep=0pt]
\item \textbf{Runs.} Every run of \method{}, each seed of each model, has a lower authored error rate than the base model on every corpus.
\item \textbf{Evaluators.} Two checkers never used in training also find more of \method{}'s authored sentences supported than the base model's, in all 20 settings under MiniCheck-Flan-T5 and in 19 under AlignScore. With no automatic evaluator at all, the human annotators find 24.8\% of \method{}'s sentences unsupported against 33.2\% of the base model's (250 ELI5 questions, Llama-3.1-8B; \(p=0.01\)).
\item \textbf{Copy test.} \method{} does not learn to copy just below its threshold: fewer of its authored sentences sit just below it than the base model's (9.2\% against 11.3\%), and with the threshold lowered to 0.4, its share of copied sentences stays below the base model's in every setting (Appendix~\ref{app:nearcopy}).
\item \textbf{Training.} \method{} does not drift into a shortcut: on every model, its copy share drops within the first quarter of training and then stays flat, and its refusals never exceed the base model's, whereas plain GRPO ends at a copy share of 72.8 (Finding~2; curves in Figure~\ref{fig:app-dynamics-all}, Appendix~\ref{app:dynamics}).
\end{itemize}

%% file: tables/tab_main.tex
\begin{tabular}{@{}lrrrrr r rrrrr r@{}}
\toprule
\multirow{2}{*}{\textbf{Method}} & \multicolumn{6}{c}{\textbf{Authored error rate (Err) $\downarrow$}} & \multicolumn{5}{c}{\textbf{Answer metrics (mean)}} & \multirow{2}{*}{\textbf{\begin{tabular}[c]{@{}c@{}}ASQA\\EM\end{tabular}}} \\
\cmidrule(lr){2-7}\cmidrule(lr){8-12}
 & ASQA & RAGT. & ClapNQ & ELI5 & Expert & Mean & Supp.\,$\uparrow$ & Copy & Words & Vac. & Ref. &  \\
\midrule
\multicolumn{13}{@{}l}{\textit{Qwen3.5-9B}} \\
Base model & 5.09 & 2.31 & 3.61 & 6.04 & 6.21 & 4.65 & 2.17 & 28.06 & 97.40 & 27.88 & 4.94 & 45.45 \\
GRPO + KL & \shortcut{3.97} & \shortcut{1.20} & \underline{3.19} & \shortcut{5.38} & \shortcut{5.85} & \shortcut{3.92} & 2.03 & 27.17 & 83.00 & 32.23 & 23.32 & 41.78 \\
Copy penalty & \shortcut{1.45} & \shortcut{1.20} & \shortcut{0.40} & \shortcut{1.13} & \shortcut{0.42} & \shortcut{0.92} & 1.22 & 6.18 & 29.48 & 82.31 & 79.65 & 20.22 \\
Sentence reward & \shortcut{3.74} & \shortcut{0.85} & \shortcut{3.09} & \shortcut{3.51} & \shortcut{4.25} & \shortcut{3.09} & 2.16 & 30.85 & 92.08 & 25.69 & 11.78 & 44.15 \\
\rowcolor{oursrow}\shadeL \method{} & \textbf{3.96} & \textbf{1.13} & \textbf{2.15} & \textbf{3.37} & \textbf{3.74} & \textbf{2.87} & \textbf{3.04} & 22.18 & 108.85 & 13.28 & 0.22 & 44.93\shadeR \\
\midrule
\multicolumn{13}{@{}l}{\textit{Qwen3.5-4B}} \\
Base model & 8.57 & 2.87 & 5.84 & 8.84 & 8.72 & 6.97 & 2.16 & 29.84 & 103.77 & 26.31 & 5.33 & 46.24 \\
GRPO + KL & \shortcut{3.17} & \shortcut{2.03} & \shortcut{2.94} & \shortcut{3.94} & \shortcut{3.49} & \shortcut{3.11} & 1.59 & 22.67 & 59.41 & 53.77 & 43.87 & 36.23 \\
Copy penalty & \shortcut{8.17} & \shortcut{3.06} & \shortcut{4.92} & \shortcut{6.73} & \shortcut{9.07} & \shortcut{6.39} & 2.30 & 19.73 & 86.35 & 36.27 & 28.13 & 36.34 \\
Sentence reward & \shortcut{4.28} & \textbf{2.22} & \shortcut{4.17} & \shortcut{5.91} & \shortcut{5.91} & \shortcut{4.50} & 2.42 & 32.85 & 108.07 & 21.98 & 4.40 & 46.42 \\
\rowcolor{oursrow}\shadeL \method{} & \textbf{6.47} & \underline{2.25} & \textbf{4.00} & \textbf{6.58} & \textbf{7.69} & \textbf{5.40} & \textbf{2.97} & 23.98 & 110.77 & 12.89 & 0.95 & 45.67\shadeR \\
\midrule
\multicolumn{13}{@{}l}{\textit{Llama-3.1-8B}} \\
Base model & 11.69 & 8.46 & 13.95 & 27.89 & 27.34 & 17.87 & 1.06 & 37.98 & 50.41 & 24.24 & 4.65 & 35.14 \\
GRPO + KL & \shortcut{6.56} & \shortcut{6.69} & \textbf{5.58} & \shortcut{21.05} & \underline{22.48} & \shortcut{12.47} & 1.10 & 41.26 & 49.41 & 25.50 & 4.03 & 34.50 \\
Copy penalty & \shortcut{8.04} & \shortcut{6.02} & \shortcut{8.75} & \shortcut{17.63} & \shortcut{20.53} & \shortcut{12.19} & 1.41 & 15.54 & 39.79 & 29.63 & 6.96 & 28.17 \\
Sentence reward & \shortcut{4.77} & \shortcut{5.06} & \shortcut{5.53} & \shortcut{18.18} & \shortcut{19.96} & \shortcut{10.70} & 1.04 & 42.58 & 49.76 & 23.03 & 4.00 & 34.25 \\
\rowcolor{oursrow}\shadeL \method{} & \textbf{6.03} & \textbf{3.94} & \underline{7.23} & \textbf{17.38} & \textbf{19.89} & \textbf{10.89} & \textbf{1.57} & 24.77 & 55.13 & 13.71 & 0.68 & 35.61\shadeR \\
\midrule
\multicolumn{13}{@{}l}{\textit{Llama-3.2-3B}} \\
Base model & 18.87 & 11.71 & 14.47 & 38.67 & 40.53 & 24.85 & 0.93 & 40.28 & 52.58 & 20.05 & 4.27 & 34.73 \\
GRPO + KL & \shortcut{14.41} & \shortcut{5.68} & \shortcut{9.95} & \shortcut{25.42} & \shortcut{28.00} & \shortcut{16.69} & 0.89 & 49.39 & 50.98 & 20.47 & 3.20 & 32.42 \\
Copy penalty & \shortcut{10.73} & \underline{8.99} & \shortcut{9.94} & \shortcut{26.76} & \shortcut{27.06} & \shortcut{16.70} & 1.23 & 18.86 & 37.82 & 28.29 & 7.53 & 25.79 \\
Sentence reward & \shortcut{11.82} & \shortcut{7.61} & \shortcut{6.31} & \shortcut{23.84} & \shortcut{24.01} & \shortcut{14.72} & 0.93 & 49.76 & 52.31 & 17.36 & 2.00 & 33.03 \\
\rowcolor{oursrow}\shadeL \method{} & \textbf{15.00} & \textbf{7.06} & \textbf{8.70} & \textbf{24.88} & \textbf{24.96} & \textbf{16.12} & \textbf{1.57} & 24.54 & 55.86 & 14.09 & 2.47 & 35.66\shadeR \\
\bottomrule
\end{tabular}

%% file: tables/tab_quality.tex
\begin{tabular}{@{}lrrrrrr@{}}
\toprule
\multirow{2}{*}{\textbf{Method}} & \multirow{2}{*}{\textbf{Reward}} & \multicolumn{2}{c}{\textbf{Against base}} & \multicolumn{2}{c}{\textbf{\method{} against it}} & \multirow{2}{*}{\textbf{Cov.}} \\
\cmidrule(lr){3-4}\cmidrule(lr){5-6}
 & & Gemma & Mistral & Gemma & Mistral & \\
\midrule
GRPO + KL & 92.7 & 43.6 & 43.4 & 56.4 & 55.9 & 42.0 \\
Copy penalty & 90.2 & 34.8 & 30.8 & 64.6 & 68.9 & 31.4 \\
Sentence reward & 92.2 & 49.5 & 48.0 & 51.5 & 51.1 & 45.1 \\
\rowcolor{oursrow}\shadeL \method{} & 90.6 & 51.0 & 49.1 &  &  & 46.8\shadeR \\
\bottomrule
\end{tabular}

%% file: tables/tab_ablation.tex
\begin{tabular}{@{}lrrrrrrrr@{}}
\toprule
\multirow{2}{*}{\textbf{Variant}} & \multicolumn{5}{c}{\textbf{Err}} & \multicolumn{3}{c}{\textbf{Mean}} \\
\cmidrule(lr){2-6}\cmidrule(lr){7-9}
 & ASQA & RAGT. & ClapNQ & ELI5 & Expert & Supp. & Copy & Vac. \\
\midrule
\rowcolor{oursrow}\shadeL \method{} & \textbf{4.0} & \underline{1.1} & 2.2 & \textbf{3.4} & \textbf{3.7} & 3.0 & 22.2 & 13.3\shadeR \\
$-$ copy test & \shortcut{2.4} & 1.5 & \shortcut{1.6} & \shortcut{3.5} & \shortcut{3.5} & 2.6 & 33.3 & 11.2 \\
$-$ vacuity test & 5.2 & \textbf{1.0} & \textbf{1.6} & 4.6 & 7.5 & 2.9 & 21.3 & 19.1 \\
Sentence reward & \shortcut{3.7} & \shortcut{0.9} & \shortcut{3.1} & \shortcut{3.5} & \shortcut{4.3} & 2.2 & 30.9 & 25.7 \\
Answer advantage & 5.4 & 2.3 & \underline{1.9} & 5.1 & 5.1 & 3.3 & 25.1 & 18.3 \\
Binary credit & \underline{4.6} & 1.7 & 2.5 & \underline{4.1} & \underline{4.4} & 3.0 & 20.9 & 12.8 \\
\bottomrule
\end{tabular}

%% file: sections/07_conclusion.tex
\section{Conclusion}
\label{sec:conclusion}

In this work, we showed that faithfulness rewards for long-form RAG pay for copying, and that common defenses only move this shortcut to shorter, emptier answers and refusals.
We then introduced \method{}, which credits only the sentences a model authors, that the question alone does not support and that the passages do, and compares each sentence with every sentence of its group.
On four models, it lowers the authored error rate while writing more grounded sentences, held-out corpora included, without these shortcuts or any loss in answer quality.
More broadly, we suggest asking of any reward what a policy can earn without the behavior it is meant to teach.
A RAG model should be credited for what it writes, not for what it copies.

%% file: sections/08_limitations.tex
\section*{Limitations}

\method{} needs no model beyond the reward's checker: its copy test matches strings, and its vacuity and support tests reuse the checker.
This keeps the method simple and lets it adopt a stronger checker as one appears, but its tests are only as reliable as that checker, which all methods in our comparison share.
The copy test is lexical, so a close paraphrase counts as authored; \method{} writes fewer such near-copies than the base model in 17 of the 20 settings (Appendix~\ref{app:nearcopy}).
We train models of up to 9B parameters on two training corpora, with one training run for each baseline; the gain holds on all four models and on the three held-out corpora, and we leave larger models to future work.

\method{} rewards the grounded sentences a model writes, not declining, so it refuses less than the base model, also on questions the passages do not answer, where the judge prefers the base model's refusal (Appendix~\ref{app:quality}).
Pairing \method{} with a grounded-refusal objective \citep{song_trustalign_2024} is a natural next step.
Finally, \method{} targets the reliability of what a model writes itself; where an application prefers verbatim quotation, plain faithfulness rewards already deliver it.

%% file: sections/A_appendix.tex
\setcounter{topnumber}{4}\setcounter{bottomnumber}{2}\setcounter{dbltopnumber}{3}\setcounter{totalnumber}{6}
\renewcommand{\topfraction}{0.9}\renewcommand{\bottomfraction}{0.6}\renewcommand{\dbltopfraction}{0.9}
\renewcommand{\textfraction}{0.08}\renewcommand{\floatpagefraction}{0.8}\renewcommand{\dblfloatpagefraction}{0.8}

\makeatletter\setlength{\@fptop}{0pt}\setlength{\@fpbot}{0pt plus 1fil}\setlength{\@dblfptop}{0pt}\setlength{\@dblfpbot}{0pt plus 1fil}\setlength{\@fpsep}{10pt plus 2pt}\setlength{\@dblfpsep}{10pt plus 2pt}\makeatother

\section{Formal Statements and Proofs}
\label{app:formal}

This section proves the formal claims of Sections~\ref{sec:analysis} and~\ref{sec:method} in the order they appear: the reward as a line in the share of copied sentences (Eq.~\ref{eq:line}), the properties of the graded credit, Proposition~\ref{prop:adv}, and the sentence form of the gradient (Eq.~\ref{eq:grad}).
Notation follows Section~\ref{sec:setup}, and \(\nabla\) stands for \(\nabla_\theta\).

\subsection{The reward is a line in the share of copied sentences}
\label{app:line}

\begin{lemma}[Reward decomposition]
\label{lem:line}
Let an answer \(y\) have \(n\) sentences, \(n_c\) of them copied and \(n_a=n-n_c\) authored (markup-only sentences, rare in our data, are set aside), and let \(s_c\) and \(s_a\) be the shares of its copied and of its authored sentences that the checker supports (either is 0 when the answer has no sentence of that kind).
With \(\rho=n_c/n\), the reward of Eq.~\ref{eq:reward} is
\begin{equation*}
R(y)\;=\;s_a+\rho\,(s_c-s_a).
\end{equation*}
\end{lemma}

\begin{proof}
Split the sum of Eq.~\ref{eq:reward} by provenance.
The copied sentences contribute \(n_c s_c\) supported sentences and the authored ones \(n_a s_a\), so
\begin{align*}
R(y)\;&=\;\frac{n_c\,s_c+n_a\,s_a}{n}\;=\;\rho\,s_c+(1-\rho)\,s_a\\
&=\;s_a+\rho\,(s_c-s_a). \qedhere
\end{align*}
\end{proof}

At fixed support rates, \(R\) is therefore affine in \(\rho\) with slope \(s_c-s_a\).
Turning one authored sentence into a copied one raises \(\rho\) by \(1/n\) and the reward by \((s_c-s_a)/n\), although no authored sentence has become more reliable.
The slope is positive whenever copied sentences are supported more often than authored ones, which holds under the reward's checker for every base model and corpus (Table~\ref{tab:app-finding1}).

\subsection{The graded credit}
\label{app:credit}

\begin{lemma}[Graded credit]
\label{lem:credit}
Let \(\kappa\in(0,1)\) and \(\phi(p)=1-\ln p/\ln\kappa\) on \([\kappa,1]\), the branch of Eq.~\ref{eq:phi} above the threshold. Then
\begin{enumerate}[label=(\alph*),nosep,leftmargin=*]
\item \(\phi(\kappa)=0\) and \(\phi(1)=1\);
\item \(\phi\) is strictly increasing and strictly concave on \([\kappa,1]\);
\item with the surprisal \(h(p)=-\ln p\), \(\phi(p)=1-h(p)/h(\kappa)\), and \(\phi\) is the only function affine in \(h\) that satisfies (a).
\end{enumerate}
\end{lemma}

\begin{proof}
(a) \(\ln 1=0\) gives \(\phi(1)=1\), and \(\ln\kappa/\ln\kappa=1\) gives \(\phi(\kappa)=0\).
(b) Since \(\ln\kappa<0\),
\begin{equation*}
\phi'(p)=-\frac{1}{p\ln\kappa}>0,
\qquad
\phi''(p)=\frac{1}{p^{2}\ln\kappa}<0 .
\end{equation*}
(c) The identity follows from \(\ln p/\ln\kappa=h(p)/h(\kappa)\).
If \(f(h)=\alpha+\gamma h\) satisfies (a), then \(f(h(1))=f(0)=\alpha=1\) and \(f(h(\kappa))=1+\gamma\,h(\kappa)=0\), so \(\gamma=-1/h(\kappa)\) and \(f(h(p))=\phi(p)\).
\end{proof}

Below the threshold Eq.~\ref{eq:phi} sets the credit to zero, so a sentence is credited exactly when the checker's probability exceeds \(\kappa\); above it, the credit grows with the checker's confidence, fastest near \(\kappa\).

\subsection{Proof of Proposition~\ref{prop:adv}}
\label{app:prop}

Fix a group of \(G\) answers and index its \(M\) sentences by \(k\), across all answers.
Their credits \(r_1,\dots,r_M\in[0,1]\) have mean \(\mu\) and standard deviation \(\sigma\), and Eq.~\ref{eq:adv} gives \(a_k=T(r_k)\) with
\begin{equation*}
T(r)\;=\;\frac{r-\mu}{\sigma+\epsilon},\qquad \epsilon>0 .
\end{equation*}
The proof rests on two facts.
\textbf{(F1)} \(T\) is one map for the whole group, whichever answer a sentence belongs to, and it is strictly increasing, with slope \(1/(\sigma+\epsilon)>0\).
\textbf{(F2)} A sentence that is copied, vacuous or unsupported has \(r_k=0\), the smallest credit in \([0,1]\): Eq.~\ref{eq:credit} zeroes copied and vacuous sentences, and Eq.~\ref{eq:phi} unsupported ones.

\begin{proof}
\emph{(i) Credit alone sets the order.}
By (F1), \(T(r_k)<T(r_l)\) if and only if \(r_k<r_l\), that is \(a_k<a_l\) if and only if \(r_k<r_l\), whether sentences \(k\) and \(l\) belong to the same answer or to different ones.

\emph{(ii) Shortcuts rank last.}
By (F2) and (i), a copied, vacuous or unsupported sentence has the smallest advantage of the group, \(a_k=-\mu/(\sigma+\epsilon)\).
If any sentence of the group earns credit, then \(\mu>0\) and this advantage is negative.
A refusal earns zero whenever the checker finds it unsupported or vacuous, as for the copy penalty's refusal in Example~1 (Appendix~\ref{app:examples}).

\emph{(iii) Average sentences earn nothing.}
Since \(M\mu=\sum_k r_k\),
\begin{equation*}
\sum_{k=1}^{M}a_k\;=\;\frac{1}{\sigma+\epsilon}\Big(\sum_{k=1}^{M}r_k-M\mu\Big)\;=\;0,
\end{equation*}
and \(a_k=0\) exactly when \(r_k=\mu\), so a sentence gains only if its credit exceeds the mean.

\emph{(iv) Length-free advantages.}
Given the group's \(\mu\) and \(\sigma\), \(a_k\) depends on the sentence only through \(r_k\), which is computed from its text, the question and the passages (Eq.~\ref{eq:credit}); neither its number of tokens nor the length of its answer enters, except through \(\mu\) and \(\sigma\), which every sentence of the group shares. Part (iv) concerns the advantage; in the gradient, each token of the sentence carries it (Lemma~\ref{lem:grad}).
\end{proof}

\para{Scope.}
The proposition concerns sentences that fail a test: it does not bound how many sentences an answer has, and a paraphrase whose copy share stays below \(\tau\) counts as authored; Appendix~\ref{app:nearcopy} measures how often such sentences occur.

\begin{remark}[GRPO's answer-level advantage]
Under GRPO with the reward \(R\), every token of answer \(i\) receives \(A^{(i)}=(R_i-\mu_R)/(\sigma_R+\epsilon)\), with \(\mu_R\), \(\sigma_R\) the group's mean and standard deviation of \(R\), and (i), (ii) and (iv) fail.
Take a group of two answers: answer~1 has three supported sentences, one of them copied, so \(R_1=1\); answer~2 has two grounded sentences and one unsupported sentence, so \(R_2=2/3\).
Then \(A^{(1)}>A^{(2)}\), and the copied sentence of answer~1 is reinforced more than the two grounded sentences of answer~2.
Deleting the unsupported sentence of answer~2 raises \(R_2\) to 1, and with it the advantage of every remaining token, so writing less pays.
\end{remark}

\subsection{The gradient as a policy gradient over sentences}
\label{app:grad}

\begin{lemma}[Sentence form of the gradient]
\label{lem:grad}
With one update per batch, the gradient of the objective in Eq.~\ref{eq:objective} at \(\theta=\theta_{\mathrm{old}}\) is
\begin{align*}
\nabla\mathcal{J}\;=\;&\frac{1}{N}\sum_{i,s}a^{(i)}_s\,\nabla\ln\pi_\theta\big(y^{(i)}_s\mid y^{(i)}_{<s}\big)\\
&-\beta\,\nabla\frac{1}{N}\sum_{i,t}\mathbb{D}_{i,t},
\end{align*}
with the prompt \((q,P)\) left implicit in the conditioning, as in Eq.~\ref{eq:grad}.
Under GRPO, the first term is \(\frac{1}{N}\sum_{i}A^{(i)}\,\nabla\ln\pi_\theta(y^{(i)})\).
\end{lemma}

\begin{proof}
Fix an answer \(i\) and drop its index.
At \(\theta=\theta_{\mathrm{old}}\) the importance ratio is \(w_t=1\), so \(\nabla w_t=w_t\,\nabla\ln\pi_\theta(y_t\mid y_{<t})=\nabla\ln\pi_\theta(y_t\mid y_{<t})\).
Every token \(t\) of sentence \(s\) carries \(\hat A_t=a_s\), and tokens outside every sentence carry zero and drop out, so grouping the tokens by sentence gives
\begin{align*}
&\sum_{t}\hat A_{t}\,\nabla\ln\pi_\theta(y_t\mid y_{<t})\\
&\qquad=\sum_{s}a_s\sum_{t\in s}\nabla\ln\pi_\theta(y_t\mid y_{<t})\\
&\qquad=\sum_{s}a_s\,\nabla\ln\pi_\theta(y_s\mid y_{<s}),
\end{align*}
where the last step is the chain rule of probability: the log-probabilities of a sentence's tokens add up to the log-probability of the sentence given everything before it.
Under GRPO every token of answer \(i\) carries the same \(A^{(i)}\), and the same grouping over the whole answer gives \(A^{(i)}\,\nabla\ln\pi_\theta(y^{(i)})\).
The KL term does not depend on the advantages and is the same for every method.
\end{proof}

\section{Setup Details}
\label{app:setup}

This section gives what is needed to reproduce the experiments: the data, the sentence tests and the metrics built on them, and the training settings.
The prompts are listed in Appendix~\ref{app:prompts}.

\begin{table}[!htb]
\centering
\footnotesize
\begin{adjustbox}{max width=\columnwidth}
\input{tables/app_data}
\end{adjustbox}
\caption{The five corpora. Train: training questions; Eval: evaluation questions for the 9B and 8B models; Eval (small): those for the 4B and 3B models; Passages: passages shown with each question.}
\label{tab:app-data}
\end{table}

\para{Data.}
\begin{itemize}[leftmargin=*,itemsep=1pt,topsep=2pt,parsep=0pt]
\item \emph{Training.} ASQA and RAGTruth are pooled for training (3{,}689 questions) and evaluated on their own evaluation sets, ALCE's 948 ASQA questions and RAGTruth's QA test sources.
\item \emph{Held out.} ClapNQ (its answerable development questions), ELI5 and ExpertQA are never seen in training. ELI5 is known to be hard to ground in retrieved documents \citep{krishna_hurdles_2021}, and ExpertQA was added only after every training run had finished, so nothing was tuned on it.
\item \emph{No overlap.} Training questions are disjoint from every evaluation set by identifier, source document and normalized question text.
\item \emph{Size.} The two smaller models are evaluated on at most 300 questions per corpus, to bound the cost of the judge.
\end{itemize}

\para{Sentence tests.}
\method{} and the metrics share three sentence tests; the reward uses only the second.
\begin{itemize}[leftmargin=*,itemsep=1pt,topsep=2pt,parsep=0pt]
\item \emph{Copy.} Answer and passages are lowercased and split into words. A word of the answer is covered if some window of five consecutive answer words containing it also occurs in the passages, and the copy share is the covered fraction. A sentence is copied if the copy share of its text, with list markers, headings and emphasis stripped, is at least \(\tau=0.5\); a sentence made only of such markup is neither copied nor authored and earns no credit.
\item \emph{Support.} A sentence is supported when the checker's probability is at least \(\kappa=0.5\). Long passages are split into chunks aligned with sentences, and the maximum over chunks is taken.
\item \emph{Vacuity.} A sentence is vacuous when the support check passes with the question in place of the passages, a context ablation in the spirit of conditional mutual information \citep{vanderpoel_cpmi_2022} and SelfCite \citep{pmlr-v267-chuang_selfcite_2025}.
\end{itemize}
Bespoke-MiniCheck-7B \citep{bespoke_minicheck_2024} is the checker of the reward and of \method{}; MiniCheck-Flan-T5-Large \citep{tang_minicheck_2024} and AlignScore \citep{zha_alignscore_2023} are the held-out checkers, used only in evaluation.

\begin{table}[!htb]
\centering
\footnotesize
\setlength{\tabcolsep}{3pt}
\begin{tabular}{@{}p{0.2\columnwidth}p{0.5\columnwidth}p{0.22\columnwidth}@{}}
\toprule
\textbf{Metric} & \textbf{Definition} & \textbf{Instrument} \\
\midrule
Err & authored sentences judged unsupported, over authored sentences & Gemma-4-31B \\
T5, AS & authored sentences supported, over authored sentences & MiniCheck-Flan-T5, AlignScore \\
Supp. & supported authored sentences per answer & MiniCheck-Flan-T5 \\
Copy & answer words inside a verbatim five-word match with the passages, over words & string match \\
Vac. & supported authored sentences also supported by the question alone, over supported authored sentences & Bespoke-MiniCheck-7B \\
Ref. & answers whose every sentence declines to answer, over answers & rule-based \\
Words & words per answer & -- \\
Reward & sentences supported, over sentences (Eq.~\ref{eq:reward}) & Bespoke-MiniCheck-7B \\
No auth. & answers without an authored sentence, over answers & string match \\
EM & short answers of ASQA found in the answer & ALCE string match \\
\bottomrule
\end{tabular}
\caption{The metrics of Tables~\ref{tab:main} and~\ref{tab:ablation} and of the appendix, on a 0--100 scale except counts per answer. Only Bespoke-MiniCheck-7B is also used in training.}
\label{tab:app-metrics}
\end{table}

\para{Metrics.}
The judge, Gemma-4-31B \citep{gemma_gemma4_2026} at full precision, reads the passages and one authored sentence (prompt in Appendix~\ref{app:prompts}); a sentence is an error when the judge rejects it outright, with a probability of ``Yes'' below \(3.5\times10^{-10}\).
Its probabilities are nearly binary (46\% of sentences below 0.001 and 42\% above 0.999), and this rule, chosen on one half of RAGTruth's human-labeled responses and tested on the other, reaches a balanced accuracy of 0.81 against the human labels, above the 0.75 of the reward's checker.
Copying counts as a shortcut even though Err leaves copied sentences out: a policy that copies more submits fewer of its own sentences to the judge, so Err can fall without any sentence becoming more reliable, which Supp.\ makes visible.
An answer refuses only when every sentence declines to answer, so a single hedge is not a refusal.

\begin{table}[!htb]
\centering
\footnotesize
\input{tables/app_hparams}
\caption{GRPO settings of every method and model.}
\label{tab:app-hparams}
\end{table}

\para{Training.}
All methods share the settings of Table~\ref{tab:app-hparams}: one epoch over the pooled training questions (922 steps, 32 answers per step), with rollouts sampled at the listed temperature; they differ only in the training signal, and Algorithm~\ref{alg:method} gives one \method{} update.
Evaluation decodes greedily with the same token budget.
For the copy penalty, \(\lambda\) is the support gap \(s_c-s_a\) of the model's own samples on the training questions under the reward's checker: 0.071 for Qwen3.5-9B, 0.087 for Qwen3.5-4B, 0.077 for Llama-3.1-8B and 0.123 for Llama-3.2-3B.
The sentence reward uses the sentence baseline of Eq.~\ref{eq:adv} with \(r_s=\ind[g(y_s,P)\ge\kappa]\) for every sentence of the answer.
\method{} is trained with two seeds on Qwen3.5-9B and three on Qwen3.5-4B, and the binary-credit ablation with three seeds on Qwen3.5-9B; every other policy is one training run.
Everything runs on eight H100 80GB GPUs: one training run on a 9B model takes about 3.5 hours on six of them, with the checker served on the remaining two, and evaluating one policy on all corpora, including the judge, takes about two hours on two.

\para{Penalty strengths.}
Neither penalty is tuned on the evaluation corpora.
The KL coefficient \(\beta=0.04\) is GRPO's default \citep{shao_grpo_2024}, and no other value removes the shortcuts: the KL-regularized objective is maximized by \(\pi^{*}(y)\propto\pi_{\mathrm{ref}}(y)\exp(R(y)/\beta)\) \citep{rafailov_dpo_2023}, which favors whatever raises \(R\) at every \(\beta>0\), so a larger \(\beta\) only shrinks the reward gain together with the shortcut.
The copy penalty's \(\lambda\) is the smallest value at which, by Eq.~\ref{eq:line}, a copied sentence earns no more than an authored one; a larger \(\lambda\) would only push the policy further from the passages, which its refusals already show.

\begin{algorithm}[!htb]
\small
\caption{One \method{} update for a prompt \((q,P)\)}
\label{alg:method}
\begin{algorithmic}[1]
\Require policy \(\pi_\theta\), reference \(\pi_{\mathrm{ref}}\), checker \(g\), thresholds \(\tau,\kappa\), group size \(G\)
\State Sample \(y^{(1)},\dots,y^{(G)}\sim\pi_{\theta_{\mathrm{old}}}(\cdot\mid q,P)\)
\For{each sentence \(y^{(i)}_s\) of each answer \(y^{(i)}\)}
  \If{\(c(y^{(i)}_s)\ge\tau\) \textbf{or} \(g(y^{(i)}_s,q)\ge\kappa\)}
    \State \(r^{(i)}_s\gets 0\) \Comment{copied or vacuous}
  \Else
    \State \(r^{(i)}_s\gets\phi\big(g(y^{(i)}_s,P)\big)\) \Comment{graded support}
  \EndIf
\EndFor
\State \(\mu,\sigma\gets\) mean, s.d.\ of all credits in the group
\State \(\hat A_{i,t}\gets(r^{(i)}_s-\mu)/(\sigma+\epsilon)\) for each \(t\in y^{(i)}_s\)
\State \(\theta\gets\theta+\eta\,\nabla_\theta\mathcal{J}(\theta)\) \Comment{Eq.~\ref{eq:objective}, step size \(\eta\)}
\end{algorithmic}
\end{algorithm}

\section{Analysis Details}
\label{app:analysis}

This section gives the evidence behind the three findings of Section~\ref{sec:analysis}, one subsection each.

\subsection{Copied versus authored sentences (Finding 1)}
\label{app:finding1}

\begin{table}[!htb]
\centering
\footnotesize
\begin{adjustbox}{max width=\columnwidth}
\input{tables/app_finding1}
\end{adjustbox}
\caption{Unsupported rate (\%) of authored (Auth.) and copied sentences in base-model samples (eight answers per question) under the reward's checker and the two held-out checkers of Appendix~\ref{app:setup}.}
\label{tab:app-finding1}
\end{table}

Under the reward's checker, authored sentences are unsupported 1.8 to 8.1 times as often as copied ones for every model and corpus, most on ELI5, where base models write most of the answer themselves.
The held-out checkers find the same gap, MiniCheck-Flan-T5 everywhere and AlignScore in all but one case, so it is not a quirk of one checker: entailment models accept hypotheses that overlap their premise \citep{mccoy_hans_2019}, and a copied sentence overlaps it entirely.
Human span labels on RAGTruth (131{,}087 sentences, no checker involved) agree: copied sentences carry errors at 1.9\% and authored ones at 12.2\%, and the error rate falls steadily with copy share, from 13.7\% in the lowest decile to 0.6\% in the highest.

\para{Within a question.}
Real samples follow Lemma~\ref{lem:line}: among the answers to the same question, the reward rises with the copy share \(c(y)\) for every model and corpus.
The pull acts even before training: picking the highest-reward answer of the eight raises the copy share for every model and corpus.

\begin{table}[!htb]
\centering
\footnotesize
\begin{adjustbox}{max width=\columnwidth}
\input{tables/app_paraphrase}
\end{adjustbox}
\caption{Authored, non-vacuous sentences of Qwen3.5-9B and Llama-3.1-8B base-model samples on ASQA and RAGTruth, by tercile of ROUGE-L similarity to their closest passage sentence: mean similarity, count, and share (\%) supported under each checker.}
\label{tab:app-paraphrase}
\end{table}

\para{Closeness to the passages.}
The preference does not hinge on the copy test's five-word cut-off: both checkers accept an authored sentence more often the closer it stays to its closest passage sentence (Table~\ref{tab:app-paraphrase}), so the reward's preference for copying is a property of the checker, not of the threshold.

\subsection{Training on the faithfulness reward (Finding 2)}
\label{app:finding2}

\begin{table}[!htb]
\centering
\footnotesize
\begin{adjustbox}{max width=\columnwidth}
\input{tables/app_finding2}
\end{adjustbox}
\caption{Policies trained on \(R\), evaluated on ASQA at the final checkpoint. T5: authored support under MiniCheck-Flan-T5; No auth.: answers without an authored sentence; other columns as in Table~\ref{tab:finding2}. DPO, copy-matched: pairs whose two answers copy equally.}
\label{tab:app-finding2}
\end{table}

\begin{figure*}[t]
\centering
\includegraphics[width=\textwidth]{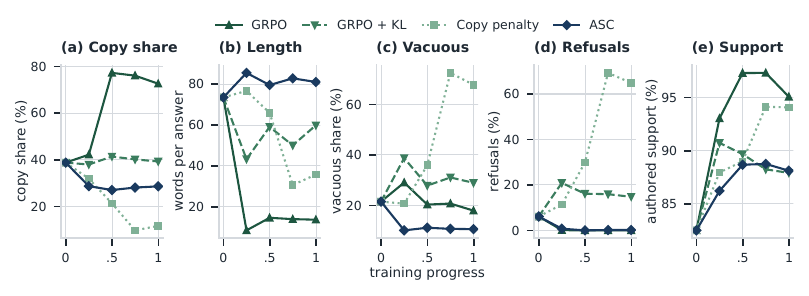}
\caption{Qwen3.5-9B on ASQA over training (0: base model): every metric of Figure~\ref{fig:finding2}, with vacuous share and authored support under MiniCheck-Flan-T5, and \method{} for comparison. Policies trained on \(R\) raise authored support by copying (GRPO) or by writing less and refusing (KL, copy penalty); \method{} copies less and writes longer answers.}
\label{fig:app-dynamics}
\end{figure*}

Every policy trained on \(R\) raises the reward through a shortcut rather than through better authored sentences (Table~\ref{tab:app-finding2}): authored support rises too, but on fewer authored sentences.
Plain GRPO copies until 81.0\% of its answers contain no authored sentence; GRPO + KL refuses and writes emptier answers; DPO copies more even when the two answers of every pair copy equally.
On Llama-3.1-8B, plain GRPO settles on a vacuous template whose most frequent sentences are ``A film was made.'' and ``A song was released.'', which earns a reward of 99.8 at an exact match of 0.5.

\para{Over training.}
Shortcuts set in early and persist (Figure~\ref{fig:app-dynamics}): GRPO's answers shrink within the first quarter of training and its copy share passes 70 by mid-training, and the copy penalty's refusals climb to 69.3\% by three quarters of training, while \method{} brings copy share and refusals down within the first quarter and keeps them there.

\para{Runs.}
Plain GRPO is GRPO + KL with \(\beta=0\).
The DPO policies are trained for one epoch on pairs of base-model samples ranked by \(R\); in the copy-matched pairs, the two answers' copy shares differ by at most 0.02 and their sentence counts by at most one.
Table~\ref{tab:finding2}, Table~\ref{tab:app-finding2} and the GRPO point of Figure~\ref{fig:teaser}b come from the evaluation pass of these runs, which decoded the base model separately from the runs of Section~\ref{sec:experiments}, so their base values differ slightly from Tables~\ref{tab:app-full} and~\ref{tab:app-shortcuts} (by at most 2.7 points, Llama-3.1-8B's vacuous share on ASQA).

\subsection{Whole-answer and sentence advantages (Finding 3)}
\label{app:finding3}

\begin{table}[!htb]
\centering
\footnotesize
\begin{adjustbox}{max width=\columnwidth}
\input{tables/app_finding3}
\end{adjustbox}
\caption{How GRPO's answer-level advantage moves each kind of sentence, on the base-model samples of Finding~1. Kinds, in the order of \method{}'s tests: copied; vacuous (authored, supportable from the question alone); grounded (authored, not vacuous, supported by the reward's checker); authored, unsupported. Count: sentences of that kind. Up, Down: the answer's reward is above, below its group's mean; Unmoved: equal to it, so the advantage is zero. Pushed up, if moved: the share pushed up among the sentences GRPO moves, pooled and its range over models and corpora.}
\label{tab:app-finding3}
\end{table}

\para{Signs of GRPO's advantage.}
Every sentence takes the sign of its answer's advantage, so Table~\ref{tab:app-finding3} counts how GRPO moves each kind of sentence.
It leaves a third of all sentences, and half of the copied ones, unmoved, because their answer's reward equals the group's mean; among the sentences it moves, it pushes up copied and grounded ones at nearly the same rate, and most vacuous ones too.
Under \method{}, a copied, vacuous or unsupported sentence always takes the lowest advantage of its group (Proposition~\ref{prop:adv}(ii)), so none is pushed up.

\para{Correlation with held-out support.}
On the same pool as Table~\ref{tab:app-paraphrase} (21{,}970 authored, non-vacuous sentences in 1{,}967 groups), we compute two advantages from the binary form of the credit in Eq.~\ref{eq:credit}: GRPO's answer-level advantage, with the credit summed over the answer, and the sentence advantage of Eq.~\ref{eq:adv}.
Their correlation with the support of the authored sentences under MiniCheck-Flan-T5 is 0.07 and 0.21: the same credit tracks a sentence's reliability far more closely when sentences are compared with sentences rather than answers with answers.
Answer length explains less than 2\% of the variance of either advantage, so the difference does not come from length.

\section{Full Results}
\label{app:full}

This section expands Table~\ref{tab:main}: every metric on every corpus, counts per answer with their significance, every training run, the robustness of the shortcut rule, near-copies, and how each model trains.
Significance marks test a difference with a paired bootstrap over questions (1{,}000 resamples, seeds pooled per question); they cover the sampling of questions, not of training runs (Appendix~\ref{app:seeds} lists the runs), and each comparison is tested on its own.

\subsection{Every metric on every corpus}
\label{app:per-corpus}

\begin{table*}[t]
\centering
\footnotesize
\setlength{\tabcolsep}{3pt}
\begin{adjustbox}{max width=\textwidth}
\input{tables/app_full}
\end{adjustbox}
\caption{Faithfulness on every corpus: authored error rate under the judge (Err; grey: a shortcut, not ranked; bold, underlined: lowest and second-lowest remaining Err, as in Table~\ref{tab:main}), authored support under MiniCheck-Flan-T5 (T5) and AlignScore (AS), and supported authored sentences per answer (Supp.); --: not scored.}
\label{tab:app-full}
\end{table*}

\begin{table*}[t]
\centering
\footnotesize
\setlength{\tabcolsep}{3pt}
\begin{adjustbox}{max width=\textwidth}
\input{tables/app_shortcuts}
\end{adjustbox}
\caption{Shortcut metrics on every corpus: copy share, words per answer, vacuous share (Vac.) and refusal rate (Ref.). A policy takes a shortcut on a corpus when its copy share, vacuous share or refusal rate exceeds the base model's by more than 3 points, or its words fall below 85\% of the base model's.}
\label{tab:app-shortcuts}
\end{table*}

\para{Faithfulness.}
The one setting in which AlignScore does not confirm the gain of Section~\ref{sec:analysis-exp} is a 0.2-point dip (Qwen3.5-9B on ELI5).
The absolute gain is largest where the base model errs most, up to 15.6 points for Llama-3.2-3B on ExpertQA, which is the same relative cut as elsewhere (38\%, Figure~\ref{fig:gain}).

\para{Shortcuts.}
The baselines' shortcuts concentrate on ELI5 and ExpertQA, the held-out corpora where the base model errs most and an answer must be written rather than assembled (Section~\ref{sec:setup}).
There GRPO + KL declines 44 to 76\% of the questions on the Qwen models, against at most 28\% on the other corpora, and the copy penalty declines 94 to 95\% on Qwen3.5-9B.
On the Llama models, GRPO + KL and the sentence reward copy more than the base model on every corpus, and the copy penalty shortens the answers instead, to 56 to 89\% of the base model's words.
\method{} does the opposite: it copies less, writes fewer vacuous sentences and refuses less than the base model on every corpus, and writes longer answers on every held-out corpus.

\subsection{Counts per answer and significance}
\label{app:counts}

\begin{table}[!htb]
\centering
\footnotesize
\setlength{\tabcolsep}{2.4pt}
\begin{adjustbox}{max width=\columnwidth}
\input{tables/app_answer_counts}
\end{adjustbox}
\caption{Base model against \method{} on every corpus (\method{}'s seeds pooled per question): authored error rate (Err), unsupported and supported authored sentences per answer under the judge (Unsupp.\,(J), Supp.\,(J)), and the share (\%) of all sentences, copied ones included, that MiniCheck-Flan-T5 does not support (Whole). $^{*}$ (on Err, Unsupp.\ and Whole): the 95\% bootstrap interval over questions of the difference from the base model excludes zero. Pooling seeds per question makes Err differ slightly from Table~\ref{tab:main}.}
\label{tab:app-counts}
\end{table}

Err is a rate over authored sentences, and \method{} writes more of them, so a lower rate could hide more errors per answer; Table~\ref{tab:app-counts} therefore counts them.
Err is lower in every setting, significantly in 14 of the 20.
Per answer, \method{} writes 0.3 to 1.0 more judge-supported authored sentences, and its unsupported ones fall in 16 settings and rise significantly in none.
Counting copied sentences too, the unsupported share of the whole answer falls in 18 settings and rises significantly in none, so answers become more faithful as a whole.
Within bins of answer length, \method{} still writes more supported authored sentences in all 20 settings, so the added content is not a by-product of length.

\subsection{Every training run}
\label{app:seeds}

\begin{table}[!htb]
\centering
\footnotesize
\begin{adjustbox}{max width=\columnwidth}
\input{tables/app_seeds}
\end{adjustbox}
\caption{Authored error rate of every \method{} training run on every corpus, with the base model above.}
\label{tab:app-seeds}
\end{table}

Table~\ref{tab:app-seeds} lists every \method{} training run separately.
Each run has a lower authored error rate than the base model on every corpus it was evaluated on, so no single seed carries the gain.

\subsection{Robustness of the shortcut rule}
\label{app:rule}

\begin{table}[!htb]
\centering
\footnotesize
\begin{adjustbox}{max width=\columnwidth}
\input{tables/app_sensitivity}
\end{adjustbox}
\caption{Table~\ref{tab:main} under other shortcut rules. Margin: points above the base model at which copy share, vacuous share or refusals flag a shortcut. \method{} flagged: settings (of 20) in which \method{} takes a shortcut. \method{} lowest Err: settings in which it has the lowest Err among policies without a shortcut, for word floors of 75\%, 85\% and 95\% of the base model's words. Baseline cells flagged: baseline Err cells (of 60) with a shortcut, floor 85\%. Table~\ref{tab:main} uses a margin of 3 and a floor of 85\%.}
\label{tab:app-sensitivity}
\end{table}

The shortcut rule of Section~\ref{sec:exp-setup} has two free choices, the margin and the word floor, and Table~\ref{tab:app-sensitivity} recomputes the ranking of Table~\ref{tab:main} under other values of both.
\method{} is flagged in no setting under any of them.
The number of settings in which it has the lowest Err falls at wider margins, mostly because they stop flagging the sentence reward, which copies a median 3.8 points more than the base model and 14.0 more than \method{}.
Without any threshold the picture is the same: in 19 of the 20 settings, no baseline has both a lower Err and at least as many supported authored sentences per answer as \method{}; the exception, the copy penalty on ASQA for Llama-3.2-3B, writes answers half as long.
Where a baseline takes no shortcut, in 5 of its 60 results, \method{}'s error rate is lower in one (Llama-3.1-8B on ExpertQA, 19.9 against 22.5 for GRPO~+~KL) and no different in the other four, while \method{} writes significantly more judge-supported authored sentences in all five (0.4 to 0.8 per answer).

\subsection{Near-copies}
\label{app:nearcopy}

A policy trained against a copy test could learn to copy just below its threshold. Section~\ref{sec:analysis-exp} rules this out for verbatim copying (just below: a copy share of 0.40 to 0.50), and with the threshold lowered to 0.4, \method{} also keeps its lower error rate.
Beyond verbatim matches, an authored sentence of \method{} is as close to its closest passage sentence as one of the base model (median ROUGE-L precision 0.42 for both; \citealp{lin_rouge_2004}), and \method{} writes fewer near-copies, with a ROUGE-L of at least 0.8, in 17 of the 20 settings.
Within terciles of closeness, its authored error rate is below the base model's in 58 of the 60 tercile and setting pairs, so its gain is not a product of closeness.

\subsection{Training dynamics}
\label{app:dynamics}

\begin{figure*}[t]
\centering
\includegraphics[width=\textwidth]{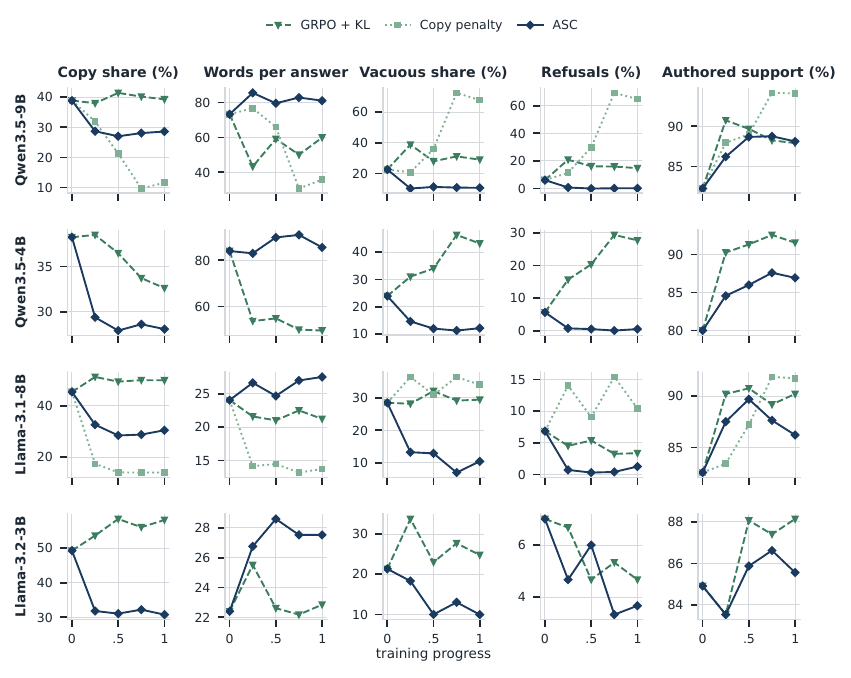}
\caption{All four models on ASQA over training (0: base model): GRPO + KL, \method{}, and the copy penalty where its intermediate checkpoints were evaluated.}
\label{fig:app-dynamics-all}
\end{figure*}

Figure~\ref{fig:app-dynamics-all} follows GRPO + KL and \method{} through training on every model, and the copy penalty on Qwen3.5-9B and Llama-3.1-8B.
The baselines' shortcuts build up over training: GRPO + KL's refusals rise from the first checkpoint on the Qwen models, up to 29\% of answers for Qwen3.5-4B, and its copy share rises on the Llama models, while the copy penalty's answers shrink as its refusals and vacuous share grow.
Their authored support ends at or above \method{}'s, but over the fewer sentences they still write.
\method{} settles early on every model: its copy share falls within the first quarter of training and then stays flat, its vacuous share falls by mid-training, and it ends with longer answers and fewer refusals than the base model.

\section{Answer Quality}
\label{app:quality}

Section~\ref{sec:main} compares whole answers with LLM judges.
This section gives the protocol, how its rubric was chosen without favoring any method, the result in every setting, a second judge, and checks that need no judge preference: reference coverage, refusals and answerable questions.

\para{Protocol.}
The judge is Gemma-4-31B, the judge of the authored error rate.
It sees the question, the passages and two answers and returns the answer that serves a reader better, or a tie, with a one-sentence reason (prompt in Appendix~\ref{app:prompts}).
Every pair is judged in both orders; a pair is a win only if both orders agree, any other outcome is a tie, and a tie counts half in the win rate.
We compare \method{} (first seed) with the base model and each baseline, and each baseline with the base model, on up to 300 questions per setting, about 5{,}400 pairs each.

\begin{table}[!htb]
\centering
\footnotesize
\begin{tabular}{@{}lccc@{}}
\toprule
\multirow{2}{*}{\textbf{Probe} (ideal value)} & \multicolumn{3}{c}{\textbf{Rubric}} \\
\cmidrule(lr){2-4}
 & 1 & \textbf{2} & 3 \\
\midrule
Copied form preferred (50) & 44.8 & 47.2 & 45.9 \\
Padded answer preferred (\(\le\)50) & 3.8 & 15.5 & 11.5 \\
Intact answer preferred (high) & 99.3 & 99.3 & 99.7 \\
\bottomrule
\end{tabular}
\caption{Probes (\%) for the three candidate rubrics, built from base-model answers only: two biases (copied form, padding) and the sensitivity to a planted error. Rubric 2 (bold) is the one used.}
\label{tab:app-probes}
\end{table}

\para{Choosing the rubric.}
LLM judges favor longer answers and a fixed position \citep{zheng_judge_2023,wang_fair_2024}, and a judge that favored quoted or long text would rank methods by style.
We wrote three rubrics and measured each on probes built from base-model answers only (Table~\ref{tab:app-probes}): the same content as a verbatim passage sentence or in the model's own words (404 pairs), an answer padded with a restated sentence (200 pairs), and an answer with one supported sentence replaced by an unsupported one (153 pairs), which a useful judge must notice.
The rule, fixed in advance, keeps the rubrics that prefer the intact answer at least 80\% of the time, takes the one closest to 50\% on the copy probe, and breaks near-ties by the smaller preference for padding.
The chosen rubric states that quoting is neither a merit nor a fault and that length is no merit, and counts a refusal as right only when the passages lack the answer; it was fixed, and its hash recorded, before any answer of a trained model was judged.

\begin{table*}[t]
\centering
\footnotesize
\setlength{\tabcolsep}{5pt}
\begin{adjustbox}{max width=\textwidth}
\input{tables/app_quality}
\end{adjustbox}
\caption{Win rate (\%) of \method{} against the base model and each baseline in every setting under Gemma-4-31B, as judged and controlled for the difference in words, and reference coverage (\%) of the base model and \method{} (--: no reference answers). Bold: the 95\% bootstrap interval over questions lies above 50 (\method{} preferred), or for coverage, \method{}'s difference from the base model lies above zero; italic: below. \method{}: first seed.}
\label{tab:app-quality}
\end{table*}

\begin{figure}[!htb]
\centering
\includegraphics[width=\columnwidth]{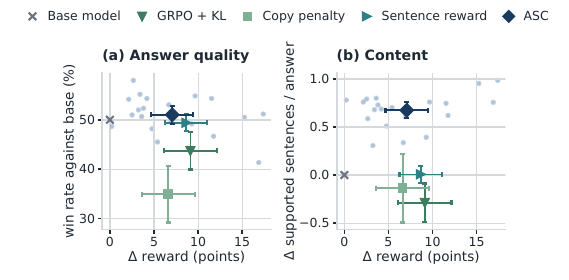}
\caption{Gain in the training reward over the base model against (a)~the win rate against the base model and (b)~the change in judge-supported authored sentences per answer. Markers: mean over the 20 settings with 95\% intervals; light points: \method{} in each setting. Every method raises the reward, but only \method{} keeps the base model's quality and adds content.}
\label{fig:app-reward-quality}
\end{figure}

\para{Preference.}
Table~\ref{tab:app-quality} gives every setting. Significant verdicts favor \method{} more often than the other method in every comparison: over the base model in 5 settings (the base model in 2), and over the copy penalty, GRPO + KL and the sentence reward in 14, 9 and 3 (each of them in at most 2).
Figure~\ref{fig:app-reward-quality} adds the spread over the 20 settings to Table~\ref{tab:quality}, and the change in supported authored sentences: only \method{} adds supported content.
Controlled for length, by a logistic fit of the verdict on the difference in words, its win rates are 48.9\% against the base model, 50.1\% against GRPO~+~KL, 54.5\% against the copy penalty and 49.2\% against the sentence reward; this control also removes the credit for the extra supported sentences, which come with more words, so it likely understates \method{}'s win rates.
A second seed on Qwen3.5-9B gives the same ordering (52.1\% and 56.4\% against the base model and GRPO + KL).

\para{Second judge.}
Mistral-Small-3.2-24B-Instruct, from a third model family, judges the same pairs with the same rubric and ranks the methods as Gemma does (Table~\ref{tab:quality}), although its copy probe shows a bias against \method{} (it prefers the copied form in 62.9\% of the pairs).
On the pairs that both judges decide, 45\% of all pairs, they agree on 81.7\% (\(\kappa\) 0.61).

\para{Reference coverage.}
ASQA, ClapNQ and ELI5 come with reference answers; the judge splits each reference into key points and checks which ones an answer conveys.
\method{} covers the most key points on average (Table~\ref{tab:quality}), significantly more than the base model in 4 of the 12 settings and less in none.
The NLI model TRUE \citep{honovich_true_2022}, which is not an LLM judge, agrees with Gemma on 90.7\% of the key points (\(\kappa\) 0.80) and gives the same direction (+1.3 points).

\para{Refusals.}
Refusing is right when the passages lack the answer, and the base model mostly refuses for that reason: the judge finds that the passages do not answer 85.4\% of its 575 refusals.
GRPO + KL, the sentence reward and \method{} also refuse mostly on such questions (79.0\%, 93.7\% and 88.8\% of their refusals), but the copy penalty does not: 47.2\% of its 3{,}821 refusals fall on questions the passages answer fully or in part.
\method{} answers most of the questions the base model rightly refuses, with an authored error rate of 5.5\%, and there the judge prefers the refusal (\method{} wins 32.3\% of 181 pairs); where both models answer, \method{} wins 51.5\%.

\begin{table}[!htb]
\centering
\footnotesize
\setlength{\tabcolsep}{3pt}
\input{tables/app_answerable_asc}
\caption{Base model against \method{} on the questions whose passages contain the answer: questions, authored error rate (Err) and supported authored sentences per answer (Supp.). $^{*}$: a significant difference from the base model (bootstrap over questions, Appendix~\ref{app:full}).}
\label{tab:app-answerable}
\end{table}

\para{Answerable questions.}
On the questions whose passages contain the answer (labeled by the judge; Mistral agrees on 84\% of 300 sampled questions), \method{} has a lower error rate than the base model in 19 of the 20 settings and writes significantly more supported authored sentences in all 20 (Table~\ref{tab:app-answerable}).

\section{Ablation Details}
\label{app:ablation}

\begin{table*}[t]
\centering
\footnotesize
\setlength{\tabcolsep}{3pt}
\begin{adjustbox}{max width=\textwidth}
\input{tables/app_ablation}
\end{adjustbox}
\caption{Ablations on Qwen3.5-9B on every corpus: authored error rate (Err; grey, bold and underlined as in Table~\ref{tab:ablation}), copy and vacuous share, and supported authored sentences per answer (Supp.).}
\label{tab:app-ablation}
\end{table*}

Each ablation of Section~\ref{sec:ablation} changes one part of \method{} on Qwen3.5-9B and keeps every other setting of Appendix~\ref{app:setup}; the answer-advantage variant sums the graded credit over the answer.
Table~\ref{tab:app-ablation} gives every corpus.
The vacuity test matters where vacuous sentences are common: without it, Err rises on ASQA, ELI5 and ExpertQA, the three corpora where the base model writes the most vacuous sentences, and falls slightly on RAGTruth and ClapNQ.
Removing the copy test or the graded credit, or using the answer-level advantage, raises Err or brings copying back on at least four of the five corpora.

\section{Human Evaluation}
\label{app:human}

Two graduate-level annotators experienced in the field label sentences and compare answers, blind to the method that wrote them.

\para{Sentence labels.}
The annotators label authored sentences as supported or unsupported by the passages.
On a judge-validity set drawn from base-model and trained answers on ASQA and ELI5, the judge agrees with people at least as often on the sentences of \method{} (90.8\%) as on the base model's (85.5\%), in both model families.
In a direct test with no judge, one random authored sentence from the base model and one from \method{} for each of 250 ELI5 questions answered by Llama-3.1-8B, 24.8\% of the sentences of \method{} are unsupported against 33.2\% of the base model's (exact McNemar \(p=0.01\)).
On 150 shared sentences, the annotators agree with a Cohen's kappa of 0.65.

\para{Answer preference.}
The annotators compare 150 pairs of answers, 75 of \method{} against the base model and 75 against GRPO + KL, drawn with a fixed seed across the four models and five corpora; the order of the two answers is random, and they follow the rubric of Appendix~\ref{app:quality}.
Counting a tie as half a win, they prefer \method{} to the base model in 56.7\% of the pairs and to GRPO + KL in 59.3\%.
They agree with a Cohen's kappa of 0.66, and on the pairs where both decide, their verdict matches Gemma's in 89.8\% and Mistral's in 88.9\%.

\section{Prompts}
\label{app:prompts}

Figure~\ref{fig:app-prompts} gives every prompt verbatim.
All models answer with the same prompt in one user turn, and thinking is disabled on the Qwen models.
The sentence judge's verdict is read as its probability of answering ``Yes''; the pairwise rubric was fixed before any trained answer was judged (Appendix~\ref{app:quality}).

\section{Example Answers}
\label{app:examples}

Examples~1--4 show one question for each model, two from training corpora and two from held-out ones, answered by all five policies of Table~\ref{tab:main}.
Each example is chosen to be representative of its setting.
Among the questions whose passages and five answers fit in one box, and which the base model answers without declining any part, we take the one whose answers best show each policy's typical behavior in that setting (Appendix~\ref{app:full}): \method{} copies no more than the base model and writes more grounded sentences, as it does at the median, and as many baselines as one question allows take their most frequent shortcut there.
Each box gives the question, every passage as the model saw it, and the five answers.
In the passages, \exsrc{highlighted text} is repeated verbatim, five words or more, by at least one answer, so the copied material is visible at a glance.
In the answers, every sentence is marked as \exC{copied}, \exS{authored and supported}, \exU{authored and unsupported} or \exV{vacuous} (supported or not under the judge; vacuous under the reward's checker), and \method{}'s answer is shaded.

In Example~1, the base model copies its first two sentences and adds an unsupported claim about the Qt toolkit; GRPO + KL keeps the two copied sentences and the sentence reward copies a third, and the copy penalty declines the question, as it does on 95.4\% of ELI5 questions for Qwen3.5-9B. \method{} conveys the same content in its own words, in four grounded sentences, with no error.
In Example~2, GRPO + KL and the sentence reward copy two and three of their three sentences, the copy penalty ends on a vacuous one, and \method{} copies one of its four sentences and grounds the other three.
In Example~3, the base model overstates the passage (California taxes no lottery winnings at all), GRPO + KL and the copy penalty decline a question the passage answers, the sentence reward opens with a copied sentence, and \method{} writes four grounded sentences, including the federal withholding rules.
In Example~4, GRPO + KL and the sentence reward shrink the answer to one copied sentence and the copy penalty to a vacuous one; \method{} writes three grounded sentences.

\clearpage
\begin{figure*}[p]
\begin{promptbox}[title={Answer prompt (all models, one user turn)}]{}
You are answering a question using only the passages provided below.\par\smallskip
Write a complete, self-contained answer in your own words, as connected prose. Use only information that is stated in the passages; do not add facts, examples, or explanations from anywhere else. If the passages do not contain the answer, say that they do not.\par\smallskip
Do not refer to the passages by number, do not write citation markers such as [1], and do not restate the question or add a preamble. Give the answer and nothing else.\par\medskip
\slot{passages}\par
\slot{question}
\end{promptbox}
\vspace{8pt}
\begin{promptbox}[title={Sentence judge behind Err (Gemma-4-31B)}]{}
Below are a source document and a single statement taken from a response written about it. Judge only whether the statement is supported by the source document: every fact it asserts must be stated in the document or follow directly from it.\par\smallskip
Answer No if the statement contradicts the document, or if it adds any detail the document does not contain---even a plausible one. Do not use outside knowledge, and do not judge whether the statement is well written or useful.\par\medskip
\#\#\# Source document\par \slot{passages}\par\smallskip
\#\#\# Statement\par \slot{sentence}\par\smallskip
Is the statement fully supported by the source document? Answer Yes or No.
\end{promptbox}
\vspace{8pt}
\begin{promptbox}[title={Pairwise judge (Gemma-4-31B and Mistral-Small-3.2-24B)}]{}
You are comparing two answers to the same question. Both were written from the passages below, which are the only information the writer had. Decide which answer serves a reader better.\par\smallskip
First check the content. Go through the claims of each answer and compare each one with the passages: is it stated in the passages or does it follow directly from them? A claim that adds a detail the passages do not contain, or that contradicts them, is an error, and errors weigh heavily. Then ask which answer gives the reader more of what the question asks for, correctly. Then, only to separate answers that are otherwise equal, prefer the clearer one.\par\smallskip
Judge what the claims say, not how they are worded. Many accurate answers restate the passages in new words; others quote them. Neither form is better: do not reward an answer for sounding close to the passages, and do not reward it for sounding original. Do not reward length by itself; a repeated or padded point adds nothing. If the passages do contain the answer, declining to answer is a failure; if they do not, saying so is correct. Ignore the order in which the answers appear. Choose ``tie'' if neither answer is better.\par\medskip
\#\#\# Question\par \slot{question}\par\smallskip
\#\#\# Passages\par \slot{passages}\par\smallskip
\#\#\# Answer A\par \slot{answer\_a}\par\smallskip
\#\#\# Answer B\par \slot{answer\_b}\par\medskip
Reply with one JSON object and nothing else: \{"reason": "<one sentence>", "verdict": "A" | "B" | "tie"\}
\end{promptbox}
\caption{The three prompts, verbatim: the answer prompt shared by all models (top), the sentence judge behind the authored error rate (middle), and the pairwise judge of Appendix~\ref{app:quality} (bottom).}
\label{fig:app-prompts}
\end{figure*}

\begin{figure*}[p]
\input{tables/app_example_1}
\end{figure*}

\begin{figure*}[p]
\input{tables/app_example_2}
\end{figure*}

\begin{figure*}[p]
\input{tables/app_example_3}
\end{figure*}

\begin{figure*}[p]
\input{tables/app_example_4}
\end{figure*}

%% file: tables/app_data.tex
\begin{tabular}{@{}llrrrr@{}}
\toprule
\textbf{Corpus} & \textbf{Role} & \textbf{Train} & \textbf{Eval} & \textbf{Eval (small)} & \textbf{Passages} \\
\midrule
ASQA & train + in domain & 2,850 & 948 & 300 & 5 \\
RAGTruth & train + in domain & 839 & 150 & 150 & 3 \\
ClapNQ & held out & -- & 300 & 300 & 1 \\
ELI5 & held out & -- & 1,000 & 300 & 5 \\
ExpertQA & held out & -- & 1,183 & 300 & 3 \\
\bottomrule
\end{tabular}

%% file: tables/app_hparams.tex
\begin{tabular}{@{}lr@{}}
\toprule
\textbf{Setting} & \textbf{Value} \\
\midrule
Answers per question $G$ & 8 \\
KL coefficient $\beta$ & 0.04 \\
Learning rate & 1e-6 \\
Warm-up steps & 36 \\
Epochs & 1 \\
Updates per batch & 1 \\
Loss & token-normalized \\
Max.\ answer tokens & 1024 \\
Sampling temperature / top-$p$ & 0.7 / 0.95 \\
Gradient clipping & 1.0 \\
Precision & bf16, full weight \\
Standardization $\epsilon$ & 1e-4 \\
\bottomrule
\end{tabular}

%% file: tables/app_finding1.tex
\begin{tabular}{@{}llrrrrrr@{}}
\toprule
\multirow{2}{*}{\textbf{Model}} & \multirow{2}{*}{\textbf{Corpus}} & \multicolumn{2}{c}{\textbf{Reward checker}} & \multicolumn{2}{c}{\textbf{MiniCheck-Flan-T5}} & \multicolumn{2}{c}{\textbf{AlignScore}} \\
\cmidrule(lr){3-4}\cmidrule(lr){5-6}\cmidrule(lr){7-8}
 &  & Auth. & Copied & Auth. & Copied & Auth. & Copied \\
\midrule
Qwen3.5-9B & ASQA & 15.3 & 4.2 & 17.4 & 6.1 & 12.2 & 6.3 \\
 & RAGTruth & 7.7 & 3.4 & 13.7 & 5.8 & 6.4 & 4.9 \\
 & ClapNQ & 3.6 & 1.1 & 10.9 & 5.0 & 3.1 & 0.7 \\
 & ELI5 & 23.9 & 3.0 & 29.2 & 8.6 & 10.7 & 8.5 \\
\midrule
Qwen3.5-4B & ASQA & 18.1 & 5.5 & 19.0 & 5.3 & 13.9 & 7.5 \\
 & RAGTruth & 8.3 & 4.5 & 14.0 & 6.6 & 7.4 & 5.0 \\
 & ClapNQ & 5.7 & 1.5 & 13.2 & 5.6 & 3.7 & 1.2 \\
 & ELI5 & 26.3 & 5.5 & 29.4 & 8.4 & 13.4 & 8.8 \\
\midrule
Llama-3.1-8B & ASQA & 18.3 & 5.8 & 19.3 & 7.4 & 16.5 & 8.0 \\
 & RAGTruth & 9.1 & 4.6 & 14.9 & 11.2 & 6.2 & 6.7 \\
 & ClapNQ & 6.0 & 2.3 & 16.7 & 7.7 & 5.6 & 0.5 \\
 & ELI5 & 24.3 & 6.0 & 29.5 & 9.0 & 22.5 & 9.1 \\
\midrule
Llama-3.2-3B & ASQA & 26.2 & 7.7 & 21.7 & 7.3 & 19.6 & 8.3 \\
 & RAGTruth & 14.7 & 4.6 & 19.5 & 8.0 & 10.7 & 4.1 \\
 & ClapNQ & 10.5 & 3.0 & 21.6 & 8.4 & 8.1 & 0.9 \\
 & ELI5 & 32.6 & 9.6 & 34.3 & 9.5 & 26.5 & 8.2 \\
\bottomrule
\end{tabular}

%% file: tables/app_paraphrase.tex
\begin{tabular}{@{}lrrrr@{}}
\toprule
\textbf{Tercile} & \textbf{ROUGE-L} & \textbf{Sentences} & \textbf{Reward checker} & \textbf{MiniCheck-Flan-T5} \\
\midrule
least similar & 0.27 & 7,362 & 74.8 & 73.3 \\
middle & 0.40 & 7,285 & 86.2 & 82.9 \\
closest & 0.57 & 7,323 & 91.4 & 89.2 \\
\bottomrule
\end{tabular}

%% file: tables/app_finding2.tex
\begin{tabular}{@{}lrrrrrrrr@{}}
\toprule
\textbf{Training} & \textbf{Reward} & \textbf{T5} & \textbf{Copy} & \textbf{Words} & \textbf{Vac.} & \textbf{Ref.} & \textbf{No auth.} & \textbf{EM} \\
\midrule
\multicolumn{9}{@{}l}{\textit{Qwen3.5-9B}} \\
Base model & 87.1 & 82.5 & 38.8 & 73.5 & 21.5 & 6.1 & 15.4 & 45.5 \\
GRPO & 98.9 & 96.1 & 72.8 & 13.7 & 18.1 & 0.0 & 81.0 & 27.6 \\
GRPO + KL & 92.3 & 87.7 & 39.2 & 60.0 & 28.5 & 14.7 & 19.5 & 41.8 \\
Copy penalty & 92.0 & 94.3 & 11.8 & 36.9 & 67.4 & 64.5 & 1.3 & 20.2 \\
DPO & 94.6 & 91.0 & 53.4 & 55.7 & 16.9 & 1.7 & 37.2 & 41.8 \\
DPO, copy-matched & 90.4 & 85.8 & 44.9 & 59.7 & 21.5 & 6.1 & 24.7 & 43.8 \\
\midrule
\multicolumn{9}{@{}l}{\textit{Llama-3.1-8B}} \\
Base model & 87.3 & 83.0 & 45.5 & 24.0 & 25.8 & 7.0 & 47.9 & 35.2 \\
GRPO & 99.8 & 99.8 & 4.6 & 7.6 & 58.9 & 0.0 & 5.2 & 0.5 \\
\bottomrule
\end{tabular}

%% file: tables/app_finding3.tex
\begin{tabular}{@{}lrrrrrr@{}}
\toprule
\multirow{2}{*}{\textbf{Sentence}} & \multirow{2}{*}{\textbf{Count}} & \multicolumn{3}{c}{\textbf{All sentences (\%)}} & \multicolumn{2}{c}{\textbf{Pushed up, if moved (\%)}} \\
\cmidrule(lr){3-5}\cmidrule(lr){6-7}
 &  & Up & Down & Unmoved & Pooled & Range \\
\midrule
Copied & 38,106 & 36.3 & 14.1 & 49.6 & 72.0 & 66.8--83.0 \\
Vacuous & 21,344 & 40.9 & 28.7 & 30.5 & 58.8 & 54.7--75.6 \\
Grounded & 72,765 & 50.1 & 16.3 & 33.6 & 75.5 & 70.1--93.3 \\
Authored, unsupported & 18,011 & 19.7 & 75.7 & 4.6 & 20.7 & 3.0--27.8 \\
\bottomrule
\end{tabular}

%% file: tables/app_full.tex
\begin{tabular}{@{}lrrrr rrrr rrrr rrrr rrrr@{}}
\toprule
\multirow{2}{*}{\textbf{Method}} & \multicolumn{4}{c}{\textbf{ASQA}} & \multicolumn{4}{c}{\textbf{RAGTruth}} & \multicolumn{4}{c}{\textbf{ClapNQ}} & \multicolumn{4}{c}{\textbf{ELI5}} & \multicolumn{4}{c}{\textbf{ExpertQA}} \\
\cmidrule(lr){2-5}\cmidrule(lr){6-9}\cmidrule(lr){10-13}\cmidrule(lr){14-17}\cmidrule(lr){18-21}
 & Err & T5 & AS & Supp. & Err & T5 & AS & Supp. & Err & T5 & AS & Supp. & Err & T5 & AS & Supp. & Err & T5 & AS & Supp. \\
\midrule
\multicolumn{21}{@{}l}{\textit{Qwen3.5-9B}} \\
Base model & 5.1 & 82.3 & 89.9 & 1.6 & 2.3 & 87.9 & 94.2 & 3.0 & 3.6 & 88.4 & 96.7 & 1.8 & 6.0 & 70.9 & 89.6 & 2.2 & 6.2 & 71.5 & 89.7 & 2.3 \\
GRPO + KL & \shortcut{4.0} & 87.7 & 91.2 & 1.4 & \shortcut{1.2} & 88.8 & 93.0 & 3.0 & \underline{3.2} & 90.6 & 97.3 & 1.8 & \shortcut{5.4} & 79.5 & 91.1 & 2.0 & \shortcut{5.9} & 81.6 & 92.9 & 2.1 \\
Copy penalty & \shortcut{1.5} & 94.3 & 95.1 & 1.4 & \shortcut{1.2} & 73.5 & 92.8 & 1.2 & \shortcut{0.4} & 87.6 & 98.0 & 1.5 & \shortcut{1.1} & 94.0 & 98.8 & 1.1 & \shortcut{0.4} & 74.5 & 98.6 & 0.9 \\
Sentence reward & \shortcut{3.7} & 87.1 & -- & 1.5 & \shortcut{0.9} & 90.9 & -- & 2.9 & \shortcut{3.1} & 90.7 & -- & 1.7 & \shortcut{3.5} & 80.0 & -- & 2.4 & \shortcut{4.3} & 80.4 & 90.0 & 2.5 \\
\rowcolor{oursrow}\shadeL \method{} & \textbf{4.0} & 87.9 & 90.7 & 2.3 & \textbf{1.1} & 90.9 & 94.6 & 3.8 & \textbf{2.2} & 90.8 & 98.3 & 2.5 & \textbf{3.4} & 78.8 & 89.4 & 3.2 & \textbf{3.7} & 84.9 & 90.7 & 3.4\shadeR \\
\midrule
\multicolumn{21}{@{}l}{\textit{Qwen3.5-4B}} \\
Base model & 8.6 & 80.0 & 87.1 & 1.7 & 2.9 & 87.9 & 92.2 & 2.9 & 5.8 & 89.6 & 96.1 & 1.8 & 8.8 & 72.7 & 88.1 & 2.3 & 8.7 & 64.7 & 85.1 & 2.1 \\
GRPO + KL & \shortcut{3.2} & 91.8 & 94.3 & 1.4 & \shortcut{2.0} & 86.3 & 93.9 & 2.3 & \shortcut{2.9} & 88.6 & 98.2 & 1.5 & \shortcut{3.9} & 90.5 & 95.0 & 1.5 & \shortcut{3.5} & 80.8 & 95.9 & 1.4 \\
Copy penalty & \shortcut{8.2} & 90.9 & -- & 1.7 & \shortcut{3.1} & 86.9 & -- & 3.2 & \shortcut{4.9} & 90.0 & -- & 2.3 & \shortcut{6.7} & 79.6 & -- & 2.2 & \shortcut{9.1} & 75.0 & 91.1 & 2.1 \\
Sentence reward & \shortcut{4.3} & 85.3 & -- & 1.7 & \textbf{2.2} & 89.7 & -- & 3.0 & \shortcut{4.2} & 91.5 & -- & 1.8 & \shortcut{5.9} & 76.8 & -- & 2.8 & \shortcut{5.9} & 76.3 & 87.1 & 2.8 \\
\rowcolor{oursrow}\shadeL \method{} & \textbf{6.5} & 87.3 & 90.5 & 2.4 & \underline{2.3} & 89.9 & 94.3 & 3.5 & \textbf{4.0} & 91.5 & 97.5 & 2.6 & \textbf{6.6} & 80.9 & 89.4 & 3.2 & \textbf{7.7} & 81.7 & 88.7 & 3.2\shadeR \\
\midrule
\multicolumn{21}{@{}l}{\textit{Llama-3.1-8B}} \\
Base model & 11.7 & 82.6 & 86.3 & 0.5 & 8.5 & 85.7 & 91.9 & 1.6 & 14.0 & 81.4 & 93.0 & 0.7 & 27.9 & 69.1 & 76.6 & 1.4 & 27.3 & 72.1 & 79.6 & 1.2 \\
GRPO + KL & \shortcut{6.6} & 90.2 & 92.1 & 0.4 & \shortcut{6.7} & 90.0 & 94.4 & 1.6 & \textbf{5.6} & 88.5 & 97.0 & 0.8 & \shortcut{21.1} & 77.1 & 82.3 & 1.3 & \underline{22.5} & 78.4 & 84.1 & 1.4 \\
Copy penalty & \shortcut{8.0} & 91.6 & 90.9 & 0.8 & \shortcut{6.0} & 87.1 & 92.8 & 2.0 & \shortcut{8.8} & 88.3 & 94.4 & 1.1 & \shortcut{17.6} & 80.2 & 84.0 & 1.5 & \shortcut{20.5} & 78.0 & 85.1 & 1.5 \\
Sentence reward & \shortcut{4.8} & 90.5 & -- & 0.4 & \shortcut{5.1} & 88.6 & -- & 1.4 & \shortcut{5.5} & 88.9 & -- & 0.8 & \shortcut{18.2} & 78.8 & -- & 1.4 & \shortcut{20.0} & 78.7 & 85.9 & 1.3 \\
\rowcolor{oursrow}\shadeL \method{} & \textbf{6.0} & 86.6 & 91.4 & 0.8 & \textbf{3.9} & 88.2 & 95.8 & 2.2 & \underline{7.2} & 90.4 & 97.9 & 1.3 & \textbf{17.4} & 80.8 & 84.8 & 1.8 & \textbf{19.9} & 81.5 & 86.8 & 1.8\shadeR \\
\midrule
\multicolumn{21}{@{}l}{\textit{Llama-3.2-3B}} \\
Base model & 18.9 & 84.9 & 81.8 & 0.5 & 11.7 & 83.8 & 89.2 & 1.2 & 14.5 & 75.4 & 91.7 & 0.6 & 38.7 & 63.3 & 72.7 & 1.4 & 40.5 & 63.6 & 72.6 & 1.0 \\
GRPO + KL & \shortcut{14.4} & 89.0 & 91.5 & 0.4 & \shortcut{5.7} & 89.8 & 90.3 & 1.1 & \shortcut{10.0} & 88.1 & 97.5 & 0.6 & \shortcut{25.4} & 77.2 & 83.2 & 1.4 & \shortcut{28.0} & 73.8 & 82.4 & 1.1 \\
Copy penalty & \shortcut{10.7} & 92.0 & -- & 0.8 & \underline{9.0} & 91.4 & -- & 1.7 & \shortcut{9.9} & 85.8 & -- & 1.0 & \shortcut{26.8} & 73.8 & -- & 1.3 & \shortcut{27.1} & 75.3 & 82.8 & 1.4 \\
Sentence reward & \shortcut{11.8} & 92.7 & -- & 0.3 & \shortcut{7.6} & 89.1 & -- & 1.1 & \shortcut{6.3} & 90.5 & -- & 0.7 & \shortcut{23.8} & 76.9 & -- & 1.4 & \shortcut{24.0} & 76.9 & 83.7 & 1.1 \\
\rowcolor{oursrow}\shadeL \method{} & \textbf{15.0} & 86.2 & 85.4 & 0.8 & \textbf{7.1} & 88.2 & 93.5 & 2.0 & \textbf{8.7} & 84.4 & 96.3 & 1.2 & \textbf{24.9} & 75.5 & 83.2 & 2.1 & \textbf{25.0} & 77.5 & 84.9 & 1.8\shadeR \\
\bottomrule
\end{tabular}

%% file: tables/app_shortcuts.tex
\begin{tabular}{@{}lrrrr rrrr rrrr rrrr rrrr@{}}
\toprule
\multirow{2}{*}{\textbf{Method}} & \multicolumn{4}{c}{\textbf{ASQA}} & \multicolumn{4}{c}{\textbf{RAGTruth}} & \multicolumn{4}{c}{\textbf{ClapNQ}} & \multicolumn{4}{c}{\textbf{ELI5}} & \multicolumn{4}{c}{\textbf{ExpertQA}} \\
\cmidrule(lr){2-5}\cmidrule(lr){6-9}\cmidrule(lr){10-13}\cmidrule(lr){14-17}\cmidrule(lr){18-21}
 & Copy & Words & Vac. & Ref. & Copy & Words & Vac. & Ref. & Copy & Words & Vac. & Ref. & Copy & Words & Vac. & Ref. & Copy & Words & Vac. & Ref. \\
\midrule
\multicolumn{21}{@{}l}{\textit{Qwen3.5-9B}} \\
Base model & 38.9 & 73.3 & 22.5 & 6.1 & 30.1 & 119.3 & 16.6 & 2.7 & 38.6 & 82.8 & 21.2 & 4.3 & 13.5 & 102.1 & 38.6 & 4.3 & 19.2 & 109.6 & 40.5 & 7.3 \\
GRPO + KL & 39.3 & 60.0 & 28.4 & 14.7 & 29.7 & 113.6 & 18.6 & 7.3 & 39.1 & 81.3 & 21.7 & 5.0 & 11.7 & 77.0 & 43.4 & 45.9 & 16.2 & 83.1 & 49.0 & 43.7 \\
Copy penalty & 11.8 & 36.9 & 67.4 & 64.5 & 5.0 & 34.3 & 81.6 & 78.7 & 11.4 & 41.3 & 72.7 & 66.0 & 0.9 & 16.3 & 95.2 & 95.4 & 1.8 & 18.7 & 94.8 & 93.7 \\
Sentence reward & 42.3 & 66.4 & 22.8 & 7.3 & 34.0 & 116.7 & 15.9 & 2.7 & 42.2 & 80.6 & 19.0 & 2.0 & 14.5 & 92.2 & 35.9 & 28.2 & 21.3 & 104.5 & 34.8 & 18.8 \\
\rowcolor{oursrow}\shadeL \method{} & 28.6 & 80.9 & 10.9 & 0.2 & 22.7 & 122.4 & 12.6 & 0.3 & 28.4 & 89.7 & 13.7 & 0.2 & 12.9 & 123.5 & 14.1 & 0.3 & 18.3 & 127.8 & 15.2 & 0.1\shadeR \\
\midrule
\multicolumn{21}{@{}l}{\textit{Qwen3.5-4B}} \\
Base model & 38.2 & 84.0 & 23.9 & 5.7 & 32.7 & 118.9 & 19.2 & 3.3 & 40.3 & 86.0 & 20.2 & 3.0 & 15.9 & 109.0 & 37.2 & 7.0 & 22.2 & 121.0 & 31.1 & 7.7 \\
GRPO + KL & 32.6 & 49.7 & 43.7 & 27.7 & 28.4 & 93.8 & 34.5 & 22.7 & 35.8 & 68.0 & 39.9 & 22.7 & 6.9 & 38.8 & 75.9 & 75.7 & 9.8 & 46.8 & 74.9 & 70.7 \\
Copy penalty & 24.4 & 58.0 & 37.9 & 28.0 & 23.2 & 115.5 & 24.0 & 10.7 & 27.7 & 82.5 & 26.4 & 12.3 & 9.8 & 83.8 & 46.4 & 48.7 & 13.7 & 92.0 & 46.8 & 41.0 \\
Sentence reward & 43.2 & 83.9 & 17.8 & 2.0 & 34.0 & 119.2 & 18.2 & 1.3 & 43.7 & 87.6 & 19.9 & 1.3 & 17.9 & 118.1 & 28.7 & 13.7 & 25.5 & 131.5 & 25.4 & 3.7 \\
\rowcolor{oursrow}\shadeL \method{} & 28.4 & 86.3 & 12.1 & 0.4 & 25.0 & 119.9 & 12.5 & 0.4 & 29.4 & 91.9 & 12.0 & 0.4 & 15.3 & 124.7 & 13.5 & 2.4 & 21.9 & 131.0 & 14.4 & 1.0\shadeR \\
\midrule
\multicolumn{21}{@{}l}{\textit{Llama-3.1-8B}} \\
Base model & 45.4 & 24.1 & 28.5 & 6.9 & 43.1 & 69.5 & 17.8 & 0.7 & 43.4 & 36.5 & 21.0 & 0.3 & 25.3 & 60.2 & 26.5 & 7.8 & 32.7 & 61.8 & 27.5 & 7.6 \\
GRPO + KL & 50.1 & 21.3 & 29.5 & 3.4 & 44.0 & 69.0 & 23.0 & 2.7 & 45.2 & 38.6 & 20.5 & 0.3 & 31.3 & 55.1 & 26.2 & 7.7 & 35.7 & 63.0 & 28.3 & 6.1 \\
Copy penalty & 14.2 & 13.7 & 34.2 & 10.4 & 22.8 & 61.5 & 21.9 & 4.0 & 17.6 & 29.9 & 22.3 & 1.0 & 9.3 & 44.4 & 36.5 & 12.1 & 13.8 & 49.5 & 33.2 & 7.3 \\
Sentence reward & 51.7 & 22.4 & 27.9 & 3.2 & 47.0 & 68.7 & 18.7 & 2.0 & 47.1 & 40.0 & 19.0 & 0.0 & 29.9 & 56.5 & 24.3 & 9.2 & 37.3 & 61.1 & 25.3 & 5.7 \\
\rowcolor{oursrow}\shadeL \method{} & 30.2 & 27.7 & 9.9 & 1.3 & 26.2 & 74.3 & 16.1 & 0.0 & 26.2 & 44.8 & 14.1 & 0.0 & 17.9 & 61.5 & 13.2 & 1.1 & 23.4 & 67.3 & 15.3 & 1.0\shadeR \\
\midrule
\multicolumn{21}{@{}l}{\textit{Llama-3.2-3B}} \\
Base model & 49.3 & 22.4 & 21.3 & 7.0 & 45.7 & 68.1 & 15.3 & 2.0 & 45.8 & 40.4 & 26.0 & 3.7 & 27.0 & 71.8 & 15.1 & 2.7 & 33.6 & 60.2 & 22.5 & 6.0 \\
GRPO + KL & 57.9 & 22.5 & 25.8 & 4.7 & 55.4 & 63.9 & 17.4 & 1.3 & 54.6 & 41.8 & 22.5 & 2.0 & 34.8 & 62.1 & 16.5 & 2.7 & 44.3 & 64.6 & 20.2 & 5.3 \\
Copy penalty & 14.7 & 12.6 & 34.5 & 10.0 & 29.4 & 60.2 & 17.8 & 2.0 & 22.0 & 26.8 & 28.9 & 3.0 & 12.4 & 42.2 & 28.1 & 12.0 & 15.9 & 47.3 & 32.2 & 10.7 \\
Sentence reward & 59.5 & 22.0 & 18.5 & 2.0 & 54.7 & 64.8 & 12.8 & 1.3 & 54.3 & 46.1 & 18.4 & 1.0 & 35.5 & 64.7 & 19.4 & 1.3 & 44.8 & 64.0 & 17.7 & 4.3 \\
\rowcolor{oursrow}\shadeL \method{} & 31.0 & 26.6 & 9.8 & 3.7 & 26.0 & 66.5 & 13.4 & 1.3 & 28.9 & 45.8 & 17.5 & 0.7 & 15.4 & 73.5 & 14.0 & 2.3 & 21.5 & 67.0 & 15.9 & 4.3\shadeR \\
\bottomrule
\end{tabular}

%% file: tables/app_answer_counts.tex
\begin{tabular}{@{}l rr rr rr rr@{}}
\toprule
\multirow{2}{*}{\textbf{Corpus}} & \multicolumn{2}{c}{\textbf{Err}} & \multicolumn{2}{c}{\textbf{Unsupp.\,(J)}} & \multicolumn{2}{c}{\textbf{Supp.\,(J)}} & \multicolumn{2}{c}{\textbf{Whole}} \\
\cmidrule(lr){2-3}\cmidrule(lr){4-5}\cmidrule(lr){6-7}\cmidrule(lr){8-9}
 & Base & \method{} & Base & \method{} & Base & \method{} & Base & \method{} \\
\midrule
\multicolumn{9}{@{}l}{\textit{Qwen3.5-9B}} \\
ASQA & 5.1 & 3.9$^{*}$ & 0.10 & 0.10\phantom{$^{*}$} & 1.8 & 2.5 & 13.4 & 10.5$^{*}$ \\
RAGT. & 2.3 & 1.1\phantom{$^{*}$} & 0.08 & 0.05\phantom{$^{*}$} & 3.4 & 4.1 & 10.2 & 8.5\phantom{$^{*}$} \\
ClapNQ & 3.6 & 2.2\phantom{$^{*}$} & 0.07 & 0.06\phantom{$^{*}$} & 2.0 & 2.7 & 9.5 & 8.8\phantom{$^{*}$} \\
ELI5 & 6.0 & 3.4$^{*}$ & 0.19 & 0.14$^{*}$ & 2.9 & 3.9 & 27.4 & 20.6$^{*}$ \\
Expert & 6.2 & 3.7$^{*}$ & 0.20 & 0.15$^{*}$ & 2.9 & 3.9 & 24.6 & 13.9$^{*}$ \\
\midrule
\multicolumn{9}{@{}l}{\textit{Qwen3.5-4B}} \\
ASQA & 8.6 & 6.5\phantom{$^{*}$} & 0.18 & 0.17\phantom{$^{*}$} & 1.9 & 2.5 & 14.6 & 10.9$^{*}$ \\
RAGT. & 2.9 & 2.3\phantom{$^{*}$} & 0.09 & 0.09\phantom{$^{*}$} & 3.2 & 3.8 & 10.4 & 9.4\phantom{$^{*}$} \\
ClapNQ & 5.8 & 4.0$^{*}$ & 0.12 & 0.11\phantom{$^{*}$} & 1.9 & 2.7 & 9.1 & 7.9\phantom{$^{*}$} \\
ELI5 & 8.8 & 6.6$^{*}$ & 0.28 & 0.26\phantom{$^{*}$} & 2.9 & 3.7 & 24.6 & 18.1$^{*}$ \\
Expert & 8.7 & 7.7\phantom{$^{*}$} & 0.28 & 0.30\phantom{$^{*}$} & 3.0 & 3.6 & 30.0 & 16.2$^{*}$ \\
\midrule
\multicolumn{9}{@{}l}{\textit{Llama-3.1-8B}} \\
ASQA & 11.7 & 6.0$^{*}$ & 0.07 & 0.05\phantom{$^{*}$} & 0.5 & 0.8 & 11.9 & 11.5\phantom{$^{*}$} \\
RAGT. & 8.5 & 3.9$^{*}$ & 0.15 & 0.10\phantom{$^{*}$} & 1.7 & 2.4 & 11.1 & 11.6\phantom{$^{*}$} \\
ClapNQ & 14.0 & 7.2$^{*}$ & 0.12 & 0.10\phantom{$^{*}$} & 0.7 & 1.3 & 13.3 & 9.1$^{*}$ \\
ELI5 & 27.9 & 17.4$^{*}$ & 0.55 & 0.38$^{*}$ & 1.4 & 1.8 & 26.1 & 17.9$^{*}$ \\
Expert & 27.3 & 19.9$^{*}$ & 0.46 & 0.43\phantom{$^{*}$} & 1.2 & 1.7 & 21.5 & 16.2$^{*}$ \\
\midrule
\multicolumn{9}{@{}l}{\textit{Llama-3.2-3B}} \\
ASQA & 18.9 & 15.0\phantom{$^{*}$} & 0.10 & 0.13\phantom{$^{*}$} & 0.4 & 0.7 & 10.2 & 11.8\phantom{$^{*}$} \\
RAGT. & 11.7 & 7.1$^{*}$ & 0.17 & 0.16\phantom{$^{*}$} & 1.3 & 2.1 & 12.2 & 11.4\phantom{$^{*}$} \\
ClapNQ & 14.5 & 8.7$^{*}$ & 0.11 & 0.13\phantom{$^{*}$} & 0.7 & 1.3 & 15.9 & 13.8\phantom{$^{*}$} \\
ELI5 & 38.7 & 24.9$^{*}$ & 0.84 & 0.69$^{*}$ & 1.3 & 2.1 & 31.6 & 23.6$^{*}$ \\
Expert & 40.5 & 25.0$^{*}$ & 0.66 & 0.57\phantom{$^{*}$} & 1.0 & 1.7 & 28.0 & 20.4$^{*}$ \\
\bottomrule
\end{tabular}

%% file: tables/app_seeds.tex
\begin{tabular}{@{}llrrrrr@{}}
\toprule
\textbf{Model} & \textbf{Run} & \textbf{ASQA} & \textbf{RAGT.} & \textbf{ClapNQ} & \textbf{ELI5} & \textbf{Expert} \\
\midrule
Qwen3.5-9B & Base model & 5.1 & 2.3 & 3.6 & 6.0 & 6.2 \\
 & Seed 1 & 4.5 & 1.5 & 2.0 & 3.3 & 4.1 \\
 & Seed 2 & 3.4 & 0.8 & 2.3 & 3.5 & 3.4 \\
\midrule
Qwen3.5-4B & Base model & 8.6 & 2.9 & 5.8 & 8.8 & 8.7 \\
 & Seed 1 & 6.6 & 2.0 & 3.5 & 6.3 & 6.5 \\
 & Seed 2 & 5.7 & 2.6 & 4.0 & 6.3 & 8.6 \\
 & Seed 3 & 7.1 & 2.1 & 4.5 & 7.1 & 7.9 \\
\midrule
Llama-3.1-8B & Base model & 11.7 & 8.5 & 14.0 & 27.9 & 27.3 \\
 & Seed 1 & 6.0 & 3.9 & 7.2 & 17.4 & 19.9 \\
\midrule
Llama-3.2-3B & Base model & 18.9 & 11.7 & 14.5 & 38.7 & 40.5 \\
 & Seed 1 & 15.0 & 7.1 & 8.7 & 24.9 & 25.0 \\
\bottomrule
\end{tabular}

%% file: tables/app_sensitivity.tex
\begin{tabular}{@{}rccccc@{}}
\toprule
\multirow{2}{*}{\textbf{Margin}} & \multirow{2}{*}{\textbf{\method{} flagged}} & \multicolumn{3}{c}{\textbf{\method{} lowest Err}} & \multirow{2}{*}{\textbf{\begin{tabular}[c]{@{}c@{}}Baseline\\cells flagged\end{tabular}}} \\
\cmidrule(lr){3-5}
 &  & 75\% & 85\% & 95\% &  \\
\midrule
1 & 0 & 20 & 20 & 20 & 59 \\
2 & 0 & 18 & 18 & 18 & 57 \\
3 & 0 & 18 & 18 & 18 & 55 \\
4 & 0 & 15 & 15 & 16 & 48 \\
5 & 0 & 15 & 15 & 16 & 43 \\
7.5 & 0 & 12 & 12 & 14 & 37 \\
10 & 0 & 8 & 8 & 12 & 29 \\
\bottomrule
\end{tabular}

%% file: tables/app_quality.tex
\begin{tabular}{@{}llrrrrrrrrrr@{}}
\toprule
\multirow{2}{*}{\textbf{Model}} & \multirow{2}{*}{\textbf{Corpus}} & \multicolumn{4}{c}{\textbf{Win rate of \method{} against}} & \multicolumn{4}{c}{\textbf{Controlled for length}} & \multicolumn{2}{c}{\textbf{Coverage}} \\
\cmidrule(lr){3-6}\cmidrule(lr){7-10}\cmidrule(lr){11-12}
 &  & Base & GRPO + KL & Copy pen. & Sent. rew. & Base & GRPO + KL & Copy pen. & Sent. rew. & Base & \method{} \\
\midrule
Qwen3.5-9B & ASQA & \textit{45.5} & 52.0 & \textbf{72.2} & 50.7 & \textit{43.9} & \textit{43.9} & \textbf{57.4} & \textit{45.2} & 50.9 & 52.0 \\
 & RAGTruth & 52.3 & 55.0 & \textbf{88.3} & 52.7 & 52.5 & 52.6 & 61.1 & 50.9 & -- & -- \\
 & ClapNQ & \textbf{54.2} & 53.0 & \textbf{78.3} & 52.7 & 52.7 & 51.4 & \textbf{59.1} & 50.7 & 89.5 & \textbf{93.5} \\
 & ELI5 & 51.2 & 51.7 & \textbf{76.5} & 53.2 & 47.4 & \textit{43.6} & 55.9 & 47.4 & 14.4 & 14.9 \\
 & ExpertQA & 50.5 & \textbf{59.5} & \textbf{83.8} & \textbf{54.8} & \textit{45.0} & 48.8 & 51.7 & 48.8 & -- & -- \\
\midrule
Qwen3.5-4B & ASQA & 49.5 & \textbf{61.8} & \textbf{59.7} & 48.8 & 49.4 & 50.9 & 54.2 & 48.5 & 50.9 & 51.2 \\
 & RAGTruth & 54.3 & \textbf{65.0} & 52.0 & 53.7 & 52.3 & 55.0 & 50.4 & 52.8 & -- & -- \\
 & ClapNQ & 51.0 & \textbf{62.7} & 52.8 & 52.0 & 49.2 & 51.3 & 50.2 & 51.1 & 91.0 & \textbf{92.6} \\
 & ELI5 & \textit{41.3} & \textbf{62.7} & \textit{44.0} & \textit{42.7} & \textit{39.4} & 53.6 & \textit{42.5} & \textit{42.1} & 14.1 & 14.9 \\
 & ExpertQA & 46.7 & \textbf{72.2} & \textbf{57.2} & 49.3 & \textit{45.7} & 53.0 & 51.1 & 50.3 & -- & -- \\
\midrule
Llama-3.1-8B & ASQA & 53.0 & \textbf{58.0} & \textbf{69.4} & \textbf{57.6} & 47.7 & 50.8 & 51.8 & 53.2 & 37.9 & 38.3 \\
 & RAGTruth & 48.7 & 49.0 & \textbf{61.4} & 48.3 & 43.7 & 43.2 & 50.5 & \textit{41.8} & -- & -- \\
 & ClapNQ & \textbf{58.0} & \textbf{57.3} & \textbf{72.8} & 52.3 & 47.5 & 50.6 & 54.8 & 46.7 & 63.6 & \textbf{71.2} \\
 & ELI5 & 49.2 & 49.8 & 51.2 & \textit{44.7} & 49.2 & 48.4 & 48.3 & \textit{43.2} & 13.1 & 13.1 \\
 & ExpertQA & 48.2 & 49.5 & \textbf{60.7} & 51.5 & 45.4 & 47.7 & \textbf{56.8} & 48.1 & -- & -- \\
\midrule
Llama-3.2-3B & ASQA & 52.0 & 53.3 & \textbf{67.9} & \textbf{57.5} & 50.8 & 51.0 & 53.6 & \textbf{55.1} & 36.6 & 37.6 \\
 & RAGTruth & 50.7 & 51.0 & 55.7 & 54.7 & 52.4 & 48.3 & 48.4 & 54.6 & -- & -- \\
 & ClapNQ & \textbf{55.2} & \textbf{56.2} & \textbf{73.8} & 49.5 & 51.4 & 51.9 & 56.3 & 48.5 & 65.2 & \textbf{70.6} \\
 & ELI5 & \textbf{54.3} & 53.2 & 54.3 & 53.2 & 54.3 & 52.1 & 50.1 & 52.0 & 12.7 & 12.0 \\
 & ExpertQA & \textbf{54.8} & 53.2 & \textbf{59.5} & 52.3 & 52.6 & 52.2 & \textbf{56.6} & 51.0 & -- & -- \\
\bottomrule
\end{tabular}

%% file: tables/app_answerable_asc.tex
\begin{tabular}{@{}l r rr rr@{}}
\toprule
\multirow{2}{*}{\textbf{Corpus}} & \multirow{2}{*}{\textbf{Questions}} & \multicolumn{2}{c}{\textbf{Err}} & \multicolumn{2}{c}{\textbf{Supp.}} \\
\cmidrule(lr){3-4}\cmidrule(lr){5-6}
 &  & Base & \method{} & Base & \method{} \\
\midrule
\multicolumn{6}{@{}l}{\textit{Qwen3.5-9B}} \\
ASQA & 663 & 3.9 & 3.3\phantom{$^{*}$} & 1.54 & 2.22$^{*}$ \\
RAGT. & 91 & 1.2 & 1.3\phantom{$^{*}$} & 3.16 & 3.96$^{*}$ \\
ClapNQ & 235 & 3.1 & 2.4\phantom{$^{*}$} & 1.83 & 2.57$^{*}$ \\
ELI5 & 248 & 3.7 & 2.5$^{*}$ & 3.08 & 3.53$^{*}$ \\
Expert & 363 & 5.2 & 3.5$^{*}$ & 3.12 & 3.82$^{*}$ \\
\midrule
\multicolumn{6}{@{}l}{\textit{Qwen3.5-4B}} \\
ASQA & 214 & 8.5 & 5.2$^{*}$ & 1.67 & 2.37$^{*}$ \\
RAGT. & 91 & 2.6 & 1.6\phantom{$^{*}$} & 3.07 & 3.75$^{*}$ \\
ClapNQ & 235 & 5.2 & 3.6\phantom{$^{*}$} & 1.87 & 2.65$^{*}$ \\
ELI5 & 70 & 7.3 & 4.3$^{*}$ & 2.74 & 3.38$^{*}$ \\
Expert & 86 & 6.3 & 4.6\phantom{$^{*}$} & 3.05 & 3.66$^{*}$ \\
\midrule
\multicolumn{6}{@{}l}{\textit{Llama-3.1-8B}} \\
ASQA & 663 & 7.1 & 3.1$^{*}$ & 0.46 & 0.78$^{*}$ \\
RAGT. & 91 & 3.9 & 1.2\phantom{$^{*}$} & 1.71 & 2.42$^{*}$ \\
ClapNQ & 235 & 9.7 & 5.2$^{*}$ & 0.75 & 1.35$^{*}$ \\
ELI5 & 248 & 10.6 & 5.8$^{*}$ & 1.56 & 1.94$^{*}$ \\
Expert & 363 & 10.3 & 6.9$^{*}$ & 1.25 & 1.78$^{*}$ \\
\midrule
\multicolumn{6}{@{}l}{\textit{Llama-3.2-3B}} \\
ASQA & 214 & 16.0 & 10.4\phantom{$^{*}$} & 0.43 & 0.75$^{*}$ \\
RAGT. & 91 & 8.4 & 4.2\phantom{$^{*}$} & 1.13 & 2.11$^{*}$ \\
ClapNQ & 235 & 10.2 & 6.9\phantom{$^{*}$} & 0.56 & 1.26$^{*}$ \\
ELI5 & 70 & 17.2 & 9.8\phantom{$^{*}$} & 1.41 & 1.93$^{*}$ \\
Expert & 86 & 15.5 & 11.8\phantom{$^{*}$} & 1.07 & 1.73$^{*}$ \\
\bottomrule
\end{tabular}

%% file: tables/app_ablation.tex
\begin{tabular}{@{}lrrrr rrrr rrrr rrrr rrrr@{}}
\toprule
\multirow{2}{*}{\textbf{Method}} & \multicolumn{4}{c}{\textbf{ASQA}} & \multicolumn{4}{c}{\textbf{RAGTruth}} & \multicolumn{4}{c}{\textbf{ClapNQ}} & \multicolumn{4}{c}{\textbf{ELI5}} & \multicolumn{4}{c}{\textbf{ExpertQA}} \\
\cmidrule(lr){2-5}\cmidrule(lr){6-9}\cmidrule(lr){10-13}\cmidrule(lr){14-17}\cmidrule(lr){18-21}
 & Err & Copy & Vac. & Supp. & Err & Copy & Vac. & Supp. & Err & Copy & Vac. & Supp. & Err & Copy & Vac. & Supp. & Err & Copy & Vac. & Supp. \\
\midrule
\multicolumn{21}{@{}l}{\textit{Qwen3.5-9B}} \\
\rowcolor{oursrow}\shadeL \method{} & \textbf{4.0} & 28.6 & 10.9 & 2.3 & \underline{1.1} & 22.7 & 12.6 & 3.8 & 2.2 & 28.4 & 13.7 & 2.5 & \textbf{3.4} & 12.9 & 14.1 & 3.2 & \textbf{3.7} & 18.3 & 15.2 & 3.4\shadeR \\
$-$ copy test & \shortcut{2.4} & 46.3 & 11.1 & 1.4 & 1.5 & 31.8 & 10.0 & 3.4 & \shortcut{1.6} & 42.2 & 11.5 & 1.9 & \shortcut{3.5} & 19.1 & 10.4 & 3.1 & \shortcut{3.5} & 26.9 & 12.7 & 3.2 \\
$-$ vacuity test & 5.2 & 26.9 & 17.2 & 2.0 & \textbf{1.0} & 24.0 & 15.9 & 3.7 & \textbf{1.6} & 26.7 & 17.8 & 2.4 & 4.6 & 11.3 & 23.4 & 3.0 & 7.5 & 17.8 & 21.4 & 3.3 \\
Sentence reward & \shortcut{3.7} & 42.3 & 22.8 & 1.5 & \shortcut{0.9} & 34.0 & 15.9 & 2.9 & \shortcut{3.1} & 42.2 & 19.0 & 1.7 & \shortcut{3.5} & 14.5 & 35.9 & 2.4 & \shortcut{4.3} & 21.3 & 34.8 & 2.5 \\
Answer advantage & 5.4 & 33.1 & 15.2 & 2.8 & 2.3 & 26.3 & 12.5 & 4.4 & \underline{1.9} & 33.4 & 15.2 & 2.9 & 5.1 & 13.0 & 24.0 & 3.2 & 5.1 & 19.9 & 24.6 & 3.3 \\
Binary credit & \underline{4.6} & 26.6 & 10.1 & 2.3 & 1.7 & 21.7 & 12.0 & 3.8 & 2.5 & 26.6 & 12.9 & 2.5 & \underline{4.1} & 12.3 & 13.7 & 2.9 & \underline{4.4} & 17.4 & 15.2 & 3.3 \\
\bottomrule
\end{tabular}

%% file: tables/app_example_1.tex
\begin{exbox}{Example 1: Qwen3.5-9B on ELI5 (held out)\\\normalfont Question: What is Linux, what does it do,  \&  why is it so popular?}{0.43}
\exlabel{Passages}
\textbf{[1]}~\textit{Top 50 Most Asked Linux Interview Questions \& Answers in 2019.} Q-25: What is “PR” Command in Linux? Q-26: What is “inode” in Linux? Q-27: What is the use of “man” command in Linux? Q-28: What Does “pipe” Do in Linux? Q-29: What Does “mkdir” Do in Linux? Q-30: What Does “ls” Do in Linux? Q-31: What Does “TR” command Do in Linux? Q-32: What are the Linux “Logs”? Q-33: How to Check “Logs” in Linux? Q-34: How to Play Games on Linux? Q-35: What is “zsh” in Linux? Q-36: Why “zsh” Should Be Used Instead of Bash? Q-37: Why “top” Command is Used in Linux? Q-38: Which “File System” is\par\smallskip
\textbf{[2]}~\textit{What is Linux? |.} What is Linux and why is it so popular? Browse the article What is Linux and why is it so popular? Every desktop computer uses an operating system. The most popular operating systems in use today are: \exsrc{Linux is an operating system} — very much like UNIX — that has become very \exsrc{popular over the last several years.} Operating systems are computer programs. An operating system \exsrc{is the first piece of software} that the computer executes when you turn the machine on. The operating system loads itself \exsrc{into memory and begins managing the resources available on the computer.} It then\par\smallskip
\textbf{[3]}~\textit{Who Made Linux and Why Is It Free?.} available on Unix. It used the Qt toolkit and was the first complete desktop environment for Linux. The Qt toolkit was under a proprietary license, so it wasn’t fully free software. This left other developers wanting an alternative. Miguel de Icaza and Federico Mena founded GNOME GNOME Explained: A Look at One of Linux's Most Popular Desktops GNOME Explained: A Look at One of Linux's Most Popular Desktops You're interested in Linux, and you've come across ``GNOME'', an acronym for GNU Network Object Model Environment. GNOME is one of the most popular open source interfaces, but what does that mean?\par\smallskip
\textbf{[4]}~\textit{Rute User's Tutorial and Exposition. D. LINUX Advocacy FAQ.} different distributions? What is the best distribution to use? Where do I get LINUX? D.4 LINUX Support Where does a person get LINUX support? My purchased software is supported; how does LINUX compete? D.5 LINUX Compared to Other Systems What is the most popular UNIX in the world? How many LINUX systems are there out there? What is the total cost of installing and running LINUX compared to a proprietary non-UNIX system? What is the total cost of installing and running a LINUX system compared to a proprietary UNIX system? How does LINUX compare to other operating systems in performance?\par\smallskip
\textbf{[5]}~\textit{Windows or Linux Hosting? Does it matter? - Small Planet Online, LLC.} The only real consideration is what type of technology your site is being developed in. For the record, Linux is almost unheard of at home, but \exsrc{is by far the most popular option for web hosting.} It’s secure, fast, and reliable. We normally steer our client towards \exsrc{open source solutions like Linux} and PHP, \exsrc{because they work well and} you will be in very very good company. Here’s a brief article from a few years ago that discusses why Linux is so darn popular for web hosting.\par\smallskip
\tcblower
\exlabel{Answers}
\noindent\textbf{Base model:} \exC{Linux is an operating system that is very much like UNIX and has become increasingly popular over the last several years.} \exC{As an operating system, it is the first piece of software a computer executes when turned on; it loads into memory and begins managing the resources available on the computer.} \exS{Linux is popular for several reasons: it is secure, fast, and reliable, making it the most popular option for web hosting.} \exU{Additionally, it is an open-source solution that allows developers to work well together, and it serves as a free alternative to proprietary software, such as the Qt toolkit which was not fully free.}\par\smallskip
\noindent\textbf{GRPO + KL:} \exC{Linux is an operating system that is very much like UNIX and has become increasingly popular over the last several years.} \exC{As an operating system, it is the first piece of software a computer executes when turned on; it loads into memory and begins managing the resources available on the machine.} \exS{Linux is popular for several reasons: it is secure, fast, and reliable, making it the most popular option for web hosting despite being rarely used in home environments.} \exS{Additionally, it is favored because it utilizes open source solutions that work well, placing users in very good company.}\par\smallskip
\noindent\textbf{Copy penalty:} \exV{The provided passages do not contain the answer to this question.}\par\smallskip
\noindent\textbf{Sentence reward:} \exC{Linux is an operating system that is very much like UNIX and has become increasingly popular over the last several years.} \exC{As an operating system, it is the first piece of software a computer executes when turned on; it loads into memory and begins managing the resources available on the machine.} \exS{Linux is popular for web hosting because it is secure, fast, and reliable, and it is considered a secure, fast, and reliable option that is by far the most popular choice for this purpose.} \exC{Additionally, developers often prefer open source solutions like Linux because they work well and are widely used.}\par\smallskip
\begin{oursbox}\noindent\textbf{\method{}:} \exS{Linux is an operating system that is very similar to UNIX and has gained significant popularity over the last several years.} \exS{As an operating system, it is the first software executed when a computer is turned on; it loads into memory and manages the available resources on the machine.} \exS{Linux is widely used for web hosting, where it is considered the most popular option because it is secure, fast, and reliable.} \exS{Its popularity is further supported by the fact that it is an open-source solution that works well with technologies like PHP, placing users in good company with many others who utilize it.}\end{oursbox}
\end{exbox}

%% file: tables/app_example_2.tex
\begin{exbox}{Example 2: Llama-3.1-8B on RAGTruth (in domain)\\\normalfont Question: symptoms and causes of ibs}{0.26}
\exlabel{Passages}
\textbf{[1]}~The cause of irritable bowel syndrome is currently unknown. \exsrc{It is thought to result from a combination of} abnormal gastrointestinal (GI) \exsrc{tract movements, increased awareness of bodily functions, and} a disruption in the \exsrc{communication between the brain and the GI tract. IBS-D} is irritable bowel syndrome with diarrhea.\par\smallskip
\textbf{[2]}~The exact cause of irritable bowel syndrome is unknown. It is \exsrc{believed to be due to a number of factors, including} alteration in the gastrointestinal \exsrc{(GI) tract motility, abnormal nervous system signals, increased sensitivity to pain, and food intolerances.} The following are risk factors thought to cause IBS: 1 \exsrc{Abnormal movements of the colon and small intestines} (too fast or slow, or too strong) 2 Hypersensitivity to pain from a full bowel or gas. 3 Food sensitivities, possibly caused by poor absorption of sugars or acids in food.\par\smallskip
\textbf{[3]}~People with IBS have symptoms that can include: 1 Diarrhea (often described as violent episodes of diarrhea) 2 Constipation. 3 \exsrc{Constipation alternating with diarrhea. Belly pains or cramps,} usually in the lower half of the belly, that get worse after meals and feel better after a bowel 1 movement. \exsrc{A lot of gas or bloating.}\par\smallskip
\tcblower
\exlabel{Answers}
\noindent\textbf{Base model:} \exS{Irritable bowel syndrome, or IBS, is a condition with unknown causes, but it is thought to result from a combination of factors.} \exC{These factors may include abnormal movements of the colon and small intestines, hypersensitivity to pain, and food intolerances.} \exS{People with IBS can experience a range of symptoms, including diarrhea, constipation, or alternating episodes of both, as well as belly pains or cramps, especially after eating, and a lot of gas or bloating.}\par\smallskip
\noindent\textbf{GRPO + KL:} \exC{Irritable bowel syndrome, or IBS, is a condition with unknown causes, thought to result from a combination of factors such as abnormal gastrointestinal movements, increased awareness of bodily functions, and disrupted communication between the brain and the GI tract.} \exC{It is also believed to be due to factors like altered GI tract motility, abnormal nervous system signals, increased sensitivity to pain, and food intolerances.} \exS{People with IBS may experience symptoms including diarrhea, constipation, or alternating episodes of both, accompanied by belly pains or cramps, especially after meals, and a lot of gas or bloating.}\par\smallskip
\noindent\textbf{Copy penalty:} \exS{Irritable bowel syndrome, or IBS, is a condition with unknown causes, but it is thought to be the result of a combination of factors, including abnormal movements of the GI tract, increased sensitivity to bodily functions, and disrupted communication between the brain and the GI tract.} \exV{Symptoms of IBS can include diarrhea, constipation, or alternating between the two, as well as belly pains or cramps, gas, and bloating.}\par\smallskip
\noindent\textbf{Sentence reward:} \exC{Irritable bowel syndrome, or IBS, is a condition with unknown causes, but it is thought to result from a combination of factors including abnormal gastrointestinal tract movements, increased awareness of bodily functions, and disrupted communication between the brain and the GI tract.} \exC{The exact cause is also believed to be due to a number of factors, including alterations in the GI tract motility, abnormal nervous system signals, increased sensitivity to pain, and food intolerances.} \exC{Symptoms of IBS can include diarrhea, constipation, constipation alternating with diarrhea, belly pains or cramps, and a lot of gas or bloating.}\par\smallskip
\begin{oursbox}\noindent\textbf{\method{}:} \exS{Irritable bowel syndrome, or IBS, is a condition with unknown causes, but it is thought to be the result of a combination of factors.} \exC{These factors may include abnormal movements of the gastrointestinal tract, increased awareness of bodily functions, and disruptions in communication between the brain and the GI tract.} \exS{IBS can also be caused by factors such as abnormal nervous system signals, hypersensitivity to pain, and food intolerances.} \exS{Symptoms of IBS can vary, but they often include diarrhea, constipation, or a combination of both, as well as belly pains or cramps, gas, and bloating.}\end{oursbox}
\end{exbox}

%% file: tables/app_example_3.tex
\begin{exbox}{Example 3: Qwen3.5-4B on ClapNQ (held out)\\\normalfont Question: what is the tax rate for lottery winnings in california}{0.36}
\exlabel{Passages}
\textbf{[1]}~\textit{California State Lottery.} All prizes for \exsrc{Fantasy 5, Daily Derby, Daily 3, Daily 4, and non-jackpot SuperLotto Plus, Mega Millions, and Powerball} prizes, are paid out in one payment, less 25 \% or 33 \% (depending upon \exsrc{the winner's tax documentation)} Federal withholding if the prize is over \$5,000. \exsrc{Merchandise prizes over \$5,000 are subject to} 33 \% Federal withholding. Scratchers tickets are generally one - payment prizes; however, some games have annuity options for payments each year, or per week. \exsrc{California does not tax California Lottery winnings, however} it taxes \exsrc{lottery winnings from other jurisdictions.} For \exsrc{SuperLotto Plus and Mega Millions} jackpots, the player may choose a single cash payout for a floating percentage of the jackpot, or an annuity. The \exsrc{SuperLotto Plus, Mega Millions, and Powerball} payment schedule are on \exsrc{a graduated basis over 30 annual payments.} Until 2005, when California joined Mega Millions, the payment choice on SuperLotto Plus had to be made when the ticket was bought. Since then, there is a 60 - day window after winning, in which the choice of cash or annuity is to be made.\par\smallskip
\tcblower
\exlabel{Answers}
\noindent\textbf{Base model:} \exU{California does not tax lottery winnings.} \exC{However, the state does tax lottery winnings from other jurisdictions.} \exS{Additionally, federal withholding applies to prizes over \$5,000, with rates of 25\% or 33\% depending on the winner's tax documentation, and merchandise prizes over \$5,000 are subject to a 33\% federal withholding rate.}\par\smallskip
\noindent\textbf{GRPO + KL:} \exV{The provided passages do not contain the answer to this question.}\par\smallskip
\noindent\textbf{Copy penalty:} \exV{The provided passages do not contain the answer to what the specific tax rate is for lottery winnings in California.}\par\smallskip
\noindent\textbf{Sentence reward:} \exC{The provided passages indicate that California does not tax California Lottery winnings.} \exS{However, the text notes that California does tax lottery winnings from other jurisdictions.}\par\smallskip
\begin{oursbox}\noindent\textbf{\method{}:} \exS{California does not impose a tax on lottery winnings won within the state, although it does tax winnings obtained from other jurisdictions.} \exS{Regarding federal withholding, prizes exceeding \$5,000 for games such as Fantasy 5, Daily Derby, Daily 3, Daily 4, and non-jackpot SuperLotto Plus, Mega Millions, and Powerball are subject to a federal withholding rate of either 25\% or 33\%, depending on the winner's tax documentation.} \exS{Merchandise prizes over \$5,000 are specifically subject to a 33\% federal withholding rate.} \exS{For the jackpots of SuperLotto Plus and Mega Millions, players may elect a single cash payout or an annuity option, with the payment schedule for these games, as well as Powerball, following a graduated basis over 30 annual payments.}\end{oursbox}
\end{exbox}

%% file: tables/app_example_4.tex
\begin{exbox}{Example 4: Llama-3.2-3B on ASQA (in domain)\\\normalfont Question: What kind of series is game of thrones?}{0.62}
\exlabel{Passages}
\textbf{[1]}~\textit{A Game of Thrones.} A Game of Thrones A Game of Thrones is the first novel in ``A Song of Ice and Fire'', \exsrc{a series of fantasy novels by} American \exsrc{author George R. R. Martin.} It was first published on August 1, 1996. The novel won the 1997 Locus Award and was nominated for both the 1997 Nebula Award and the 1997 World Fantasy Award. The novella ``Blood of the Dragon'', comprising the Daenerys Targaryen chapters from the novel, won the 1997 Hugo Award for Best Novella. In January 2011 the novel became a ``New York Times'' Bestseller and reached \#1 on the list\par\smallskip
\textbf{[2]}~\textit{Game of Thrones.} in a fantasy setting of magic and dragons. In a 2012 study of deaths per episode, it ranked second out of 40 recent U.S. TV drama series (with an average of 14). The series is generally praised for what is perceived as a sort of medieval realism. George R.R. Martin set out to make the story feel more like historical fiction than contemporary fantasy, with less emphasis on magic and sorcery and more \exsrc{on battles, political intrigue, and} the characters, believing that magic should be used moderately in the epic fantasy genre. Martin has stated that ``the true horrors of\par\smallskip
\textbf{[3]}~\textit{Game of Thrones (season 1).} The majority of reviews for the first season were very positive, with critics noting the high production values, the well-realized world, compelling characters, and giving particular note to the strength of the child actors. Tim Goodman's review for ``The Hollywood Reporter'' stated, ``a few minutes into HBO's epic Game of Thrones series, it's clear that the hype was right and the wait was worth it''. Mary McNamara from the ``Los Angeles Times'' called it ``...a great and thundering series of political and psychological intrigue bristling with vivid characters, cross-hatched with tantalizing plotlines and seasoned with a splash of fantasy''. ``New\par\smallskip
\textbf{[4]}~\textit{Game of Thrones.} in the pilot, playing a character described as ``a charismatic socialite hiding a dark secret.'' Game of Thrones Game of Thrones is an American fantasy drama television series created by David Benioff and D. B. Weiss. It is an adaptation of ``A Song of Ice and Fire'', George R. R. Martin's series of fantasy novels, the first of which is ``A Game of Thrones''. It is filmed in Belfast and elsewhere in Northern Ireland, Canada, Croatia, Iceland, Malta, Morocco, Scotland, Spain, and the United States. The series premiered on HBO in the United States on April 17, 2011, and its\par\smallskip
\textbf{[5]}~\textit{A Game of Thrones (board game).} A Game of Thrones (board game) A \exsrc{Game of Thrones is a} strategy board game created by Christian T. Petersen and released by Fantasy Flight Games in 2003. The game is based on the ``A Song of Ice and Fire'' fantasy series \exsrc{by George R. R. Martin.} It was followed in 2004 by the expansion ``A Clash of Kings'', and in 2006 by the expansion ``A Storm of Swords''. ``A Game of Thrones'' allows the players to take on the roles of several of the Great Houses vying for control of the Seven Kingdoms, including House Stark, House Lannister, House\par\smallskip
\tcblower
\exlabel{Answers}
\noindent\textbf{Base model:} \exS{A fantasy series set in a medieval-like world with magic and dragons, with a focus on battles, political intrigue, and character development, and an adaptation of a series of fantasy novels by George R.R. Martin.}\par\smallskip
\noindent\textbf{GRPO + KL:} \exC{Game of Thrones is a fantasy series.}\par\smallskip
\noindent\textbf{Copy penalty:} \exV{The series is a fantasy series.}\par\smallskip
\noindent\textbf{Sentence reward:} \exC{Game of Thrones is a fantasy series.}\par\smallskip
\begin{oursbox}\noindent\textbf{\method{}:} \exS{The series is a fantasy series with elements of historical fiction, and it is set in a medieval-like world with magic and dragons.} \exS{The author, George R.R. Martin, aimed to create a more realistic story with a focus on battles, politics, and character development, rather than relying heavily on magic.} \exS{The series has been adapted into a television show, which was praised for its production values, characters, and plotlines.}\end{oursbox}
\end{exbox}